%% file: main.tex
\pdfoutput=1
\documentclass[Afour,sageh,times]{sagej}

\usepackage[colorlinks,bookmarksopen,bookmarksnumbered,citecolor=blue,urlcolor=blue,linkcolor=blue]{hyperref}

\usepackage{graphicx}

\usepackage{xcolor}
\usepackage{makecell}

\input{./macros}
\begin{document}

\runninghead{Borquez et al.}

\title{Hamilton-Jacobi Reachability for Hybrid Systems: Unified Goal-Driven Control with Safety Guarantees}

\author{Javier Borquez\affilnum{1}, Shuang Peng\affilnum{2}, and Somil Bansal\affilnum{3}}

\affiliation{\affilnum{1}Universidad de Santiago de Chile
\affilnum{2}University of Southern California 
\affilnum{3}Stanford University}

\corrauth{Javier Borquez, Universidad de Santiago de Chile, Chile.}
\email{javier.borquez@usach.cl}

\begin{abstract}{
Hybrid dynamical systems provide a powerful modeling framework for robotic systems, particularly in contact-rich environments. However, ensuring safety and performance in such systems remains challenging due to the intricate coupling between continuous dynamics and discrete mode transitions.
In this work, we extend classical Hamilton-Jacobi (HJ) reachability analysis, a formal verification method for continuous-time nonlinear systems, to hybrid dynamical systems. 
Our framework characterizes safe sets for hybrid systems through a generalized value function defined over both discrete and continuous states while accounting for control constraints and model uncertainty.
We additionally provide a numerical algorithm to compute this value function.

Building on these safe sets, we propose two different mechanisms to integrate performance objectives.
First, we introduce a hybrid least-restrictive safety filter that intervenes on both the discrete and continuous components of a nominal controller only when necessary to avoid unsafe states, thereby preserving nominal behavior whenever possible. 
Second, we formulate and compute hybrid backward reach-avoid tubes, enabling the simultaneous enforcement of safety and goal-reaching behavior, an extension not previously addressed within hybrid HJ reachability. 
This enables the synthesis of continuous and discrete control policies that guarantee both safety and task completion. 
We validate our framework through simulation studies and real-world experiments on a quadrupedal robot, demonstrating its effectiveness in hybrid mode planning and safety-critical applications.}
\end{abstract}


\maketitle
\input{01_intro.tex}
\input{02_related_work.tex}
\input{03_problem_new.tex}

\input{04a_background.tex}
\input{04b_prior_work}
\input{05a_theory_LR.tex}
\input{05b_cases_LR.tex}

\input{06a_theory_BRAT}
\input{06b_cases_BRAT}

\input{10_hw_exp.tex}
\input{11_conclusion.tex}

\section{Statements and Declarations}
\begin{funding}
This work is supported in part by Becas Chile, the NSF CAREER Program under award 2240163, and the DARPA ANSR program.
\end{funding}

\begin{sm}
Hardware Experiments video:\\
https://doi.org/10.6084/m9.figshare.28950725.v2
\end{sm}

\bibliographystyle{SageH}
\bibliography{  Bib/bib_cbf,
                Bib/bib_hybrid,
                Bib/bib_local,
                Bib/bib_MPC,
                Bib/bib_reach,
                Bib/bib_SIAL}
\newpage

\end{document}

%% file: macros.tex
\usepackage{graphicx,caption}
\usepackage{amsmath,amssymb,mathtools}
\usepackage[ruled,vlined,titlenotnumbered]{algorithm2e} 
\usepackage{array,multirow}
\usepackage{xcolor}

\usepackage[capitalise]{cleveref}

\crefname{equation}{}{} 
\crefname{section}{Sec.}{Sec.}

\newcommand{\argmin}{\operatornamewithlimits{argmin}}

\newtheorem{remark}{Remark}

\newtheorem{theorem}{Theorem}

%% file: 01_intro.tex
\section{Introduction}

Hybrid dynamical systems are a popular and versatile tool to model robotic systems exhibiting continuous and discrete dynamics \citep{goebel2009hybrid, johnson2016hybrid}. They are “hybrid” because they contain continuous and discrete state variables, wherein the continuous flow is interleaved with discrete transitions. For instance, a jumping robot exhibits continuous dynamics as it pushes against the ground or jumps through the air, but undergoes discrete transitions when it launches or lands. However, their complexity and richness present challenges, particularly when it comes to ensuring the safety and performance of such systems. 

Controllers for such systems demand simultaneous reasoning about the continuous evolution of states and discrete mode transitions, a task conceptually and computationally demanding \citep{kong2021salted}. To safely leap over an obstacle, for example, a robot must decide not only whether to jump, but precisely when to engage the jump and at what velocity, such that the resulting trajectory clears the obstacle without midair collision. This requires careful coordination between continuous control and discrete decisions.

Hamilton-Jacobi (HJ) reachability analysis is a powerful tool for understanding and controlling hybrid dynamical systems. It characterizes the Backward Reachable Tube (BRT), the set of all initial states from which the system is guaranteed to enter a failure set, regardless of the control strategy. This BRT defines the system’s unsafe set, while its complement corresponds to the safe region, from which the system can avoid failure under appropriate control. Beyond identifying this safe set, HJ reachability also synthesizes a corresponding safety controller that ensures the system remains within the safe region over time.

The utility of HJ reachability lies in its ability to handle general nonlinear dynamics, control bounds, and dynamic uncertainty during BRT computation \citep{mitchell2005time}.
Accordingly, several methods have been developed to theoretically characterize and numerically compute BRTs for continuous-time dynamical systems; see \cite{bansal2017hamilton} for a comprehensive survey.
However, analogous frameworks for hybrid dynamical systems remain limited. These limitations—ranging from structural assumptions to restricted applicability—are discussed in detail in Section~\ref{related work}.

Our earlier work~\citep{borquez2024hybridreach} addressed key limitations in applying HJ reachability to hybrid systems by introducing a unified framework for computing BRTs in systems with nonlinear continuous dynamics, controlled and forced discrete transitions, state reset maps, and bounded disturbances. It presented a generalized hybrid Hamilton-Jacobi-Isaacs Variational Inequality (HJI-VI) to calculate the value function that captures the BRT and synthesizes the optimal continuous and discrete control strategies. However, that framework focused solely on computing hybrid BRTs for reach or avoid verification, and did not address how to integrate these guarantees with other control objectives, in particular the critical requirement for real-world applications where systems must not only stay safe, but also achieve goals. This work addresses that gap through two complementary contributions: a hybrid safety filter that can wrap any goal-oriented but potentially unsafe nominal policy and guarantee safety, and a hybrid reach-avoid formulation that formally handles both objectives jointly.

\subsection{\label{contributions}Paper Organization and Contributions}
 
The paper is organized as follows. Section~\ref{related work} surveys related work. Section~\ref{problem} introduces the hybrid system model and problem formulations. Section~\ref{background} reviews the needed background, with Subsection~\ref{background_hj} covering classical HJ reachability, and Subsection~\ref{prior_work} providing an abridged presentation of our prior hybrid HJ reachability framework~\citep{borquez2024hybridreach} for completeness. \textbf{All technical contributions of this paper begin in Section~}\ref{proof_LR}\textbf{ and beyond.} Specifically, Section~\ref{proof_LR} introduces the hybrid Least Restrictive Filter, Section~\ref{hybrid_BRAT} presents the hybrid Backward Reach-Avoid Tube formulation, Section~\ref{cases} presents validation on a real-world quadrupedal robot, and Section~\ref{conclusion} concludes and points toward future directions.

The two contributions of this paper are as follows:

\begin{itemize}
\item \textbf{Hybrid Least Restrictive Filter (hLRF)} (Section~\ref{proof_LR}): We introduce a real-time safety filtering mechanism for hybrid systems that wraps any nominal hybrid control policy and intervenes on both the continuous input and discrete transition decisions only when necessary to prevent entry into the unsafe set. The key technical challenge is providing formal safety guarantees in the presence of state resets, forced transitions, and mode-dependent dynamics, all of which are absent in continuous safety filters. We prove that the hLRF ensures forward invariance of the hBRT-defined safe set for \textit{any} nominal hybrid controller, and validate it through simulation studies.
\item \textbf{Hybrid Backward Reach-Avoid Tubes (hBRATs)} (Section~\ref{hybrid_BRAT}): We introduce a reach-avoid formulation for hybrid systems that simultaneously enforces safety and goal-reaching under worst-case disturbances and across discrete transitions. Jointly encoding both objectives in the hybrid setting requires fundamentally different mathematical structure than either in isolation, so we derive and prove a new constrained HJI-VI whose solution characterizes the hBRAT: the set of states from which the system is guaranteed to reach a target while avoiding unsafe states. This yields provably safe and goal-complete continuous and discrete control policies, which we verify through hardware experiments on a quadrupedal robot.
\end{itemize}
Together, these contributions bridge the gap between safety verification and goal-directed hybrid control, enabling both online safety enforcement and offline synthesis of provably safe, task-complete control policies.

%% file: 02_related_work.tex
\section{\label{related work}Related Work}
Reachability analysis for hybrid dynamical systems has been extensively studied by both computer science and control communities. 
Correspondingly, a number of different approaches have been proposed in the literature. 
In this section, we provide a brief overview of some of the prominent approaches, as well as the current research gaps.
Historically, one of the most well-studied methods for hybrid reachabiilty analysis is rooted in hybrid automata theory, such as timed automata \citep{alur1994theory} and linear hybrid automata \citep{alur1991hybrid}.
Reachability computations in these methods are typically based on propagation of polygonal sets under constant rate dynamics. 
Tools for automatically performing these computations have also been developed \citep{henzinger1995user, maler1995synthesis, yovine1997kronos}.
However, these methods typically impose restrictive assumptions on the underlying continuous dynamics, such as limiting the analysis to linear dynamics and not allowing continuous control inputs, restricting their direct use in robotics applications. 

Other classes of approaches extend reachability tools for continuous state and time dynamical systems to incorporate discrete switches \citep{ altin2020semicontinuity, chai2018forward, girard2013computational}. 
These approaches include zonotopes-based methods, computability theory, Taylor models, satisfiability modulo theory, Hamilton-Jacobi reachability, among others. 
For instance, zonotopes-based methods represent reachable sets as zonotopes or a mixture of zonotopes, and solve the hybrid reachability problem by propagating these sets and considering their interaction with discrete event transitions modeled as guard sets \citep{ zono_avoid_intersect_2012, zono_poly_2010, contact_zono_2023, nonlin_zono_2015}.
However, the efficient computation of reachable sets for
hybrid systems with nonlinear dynamics remains a difficult problem to solve.
Furthermore, it is challenging to account for controlled transitions in these methods, which is often a key requirement in the motion and trajectory planning of hybrid systems such as legged robots (e.g., where the system might want to transition between different discrete gaits in order to safely reach its goal).

Other reachability methods for hybrid systems rely on rigorous computable analysis theory to represent reachable sets as geometric objects \citep{ariadne_denotable_reach08,finite_time_denotable_2005}.
Taylor models have also been used in the analysis of hybrid reachable sets to provide rigorous enclosures of the set trajectories, while accounting for uncertainties and errors in the computation \citep{taylor_hyb_flow12,flow_star_2013,taylor_nonlin_guard_20}.
Satisfiability Modulo Theory (SMT) has also been used for the reachability analysis of hybrid systems.
Such methods encode dynamics and discrete mode transitions as first-order formulas over real numbers that are solved using an SMT solver \citep{large_discrete_2007, delta_reach2015}.
However, above methods typically compute an over-approximations of reachable sets, while limiting the discrete transitions to forced transitions.

Another approach for computing BRTs for hybrid dynamical systems is via HJ reachability analysis \citep{tomlin1996hybrid, lygeros1998controller, lygeros1996hierarchical}. Its advantages include compatibility with general non-linear system dynamics, formal treatment of bounded disturbances, and the ability to deal with state and input constraints \citep{bansal2017hamilton}. 
Several classical and modern works have addressed the control and safety analysis of hybrid dynamical systems through HJ reachability \citep{coll_avoid_HJI_hybrid_2017,drone_backflip_2011,air_3modes_200l,air_7modes_1999,ROA_reset_2022}.
These methods rely on an iterative algorithm to compute the BRT, wherein the BRT is iteratively refined in each discrete mode (using a continuous reach-avoid operator) based on the last computed BRT and the discrete predecessor maps. Intuitively, these predecessor maps capture the effect of forced and controlled transitions on the BRT of the system. This process is repeated until the BRT reaches a fixed point and converges. 
However, the BRT computation might require several iterations to converge and thus can be time-consuming \citep{air_3modes_200l}.

Recent advancements address some limitations of HJ reachability by generalizing the framework to account for discontinuous state changes during transitions \citep{ROA_reset_2022}, which inspired our preliminary work \citep{borquez2024hybridreach} on hybrid systems with multiple discrete modes, controlled transitions, forced transitions, and state resets. However, a critical challenge remains: enabling controllers to not only ensure safety but also achieve goal-oriented behavior. Real-world systems require unified frameworks that reason about reaching desired states while avoiding unsafe ones, which has been studied by non-reachability-based hybrid safety approaches, most prominently as hybrid control barrier functions, which hold close relation to HJ reachability~\citep{cbvf}, but collectively focus on verifying safety through candidate barrier functions—often locally defined, learned from data, or tailored to specific hybrid formulations~\citep{HCBF_filter,HCBF_learn,HCBF_inclusion}—rather than constructing global safe sets or synthesizing control policies that unify safety and goal-reaching guarantees.

%% file: 03_problem_new.tex
\section{\label{problem}Problem Formulation}
We consider a hybrid dynamical system as defined by ~\cite{lygeros2008hybrid}:
\begin{equation}\label{eq:hyb_def}
H=((Q \times X),(U \times D), f, \operatorname{Inv}, \Sigma, R),
\end{equation}
where $Q := \{q_1, q_2, \hdots, q_N\}$ is a finite set of discrete modes. We also refer to $q_i$ as the discrete state of the system.
Let $x \in X \subset \mathbb{R}^{n_x}$ be the continuous state of the system, $u \in U \subset \mathbb{R}^{n_u}$ be the continuous control input, and $d \in D \subset \mathbb{R}^{n_d}$ be the continuous disturbance in the system.
$f: Q \times X \times U \times D \to \mathbb{R}^{n}$ defines the continuous evolution of the system for each $q \in Q$. 
$\operatorname{Inv} \subseteq (Q \times X)$ is the invariant of each discrete state, and defines the set of states for which
continuous evolution is allowed.
$\Sigma$ is a finite set of discrete actions.
$R: Q \times X \times \Sigma \to 2^{Q \times X}$ is a reset relation, which encodes the discrete transitions of the hybrid system.

For simplicity, in mode $q_i \in Q$, we denote the continuous dynamics as $f_i$ and the invariant as $S_i$.
In other words, in each discrete mode $q_i \in Q$, the continuous state evolves according to the dynamics: $\dot{x} = f_i(x, u, d)$,
with $x \in S_i \coloneq \operatorname{Inv}(q_i) \subseteq X$, $u \in U$ and $d \in D$. 
Here, $S_i$ can be thought of as the valid operation domain for mode $q_i$.
\begin{figure}[t]
\begin{center} 
\includegraphics[width=0.95\columnwidth]{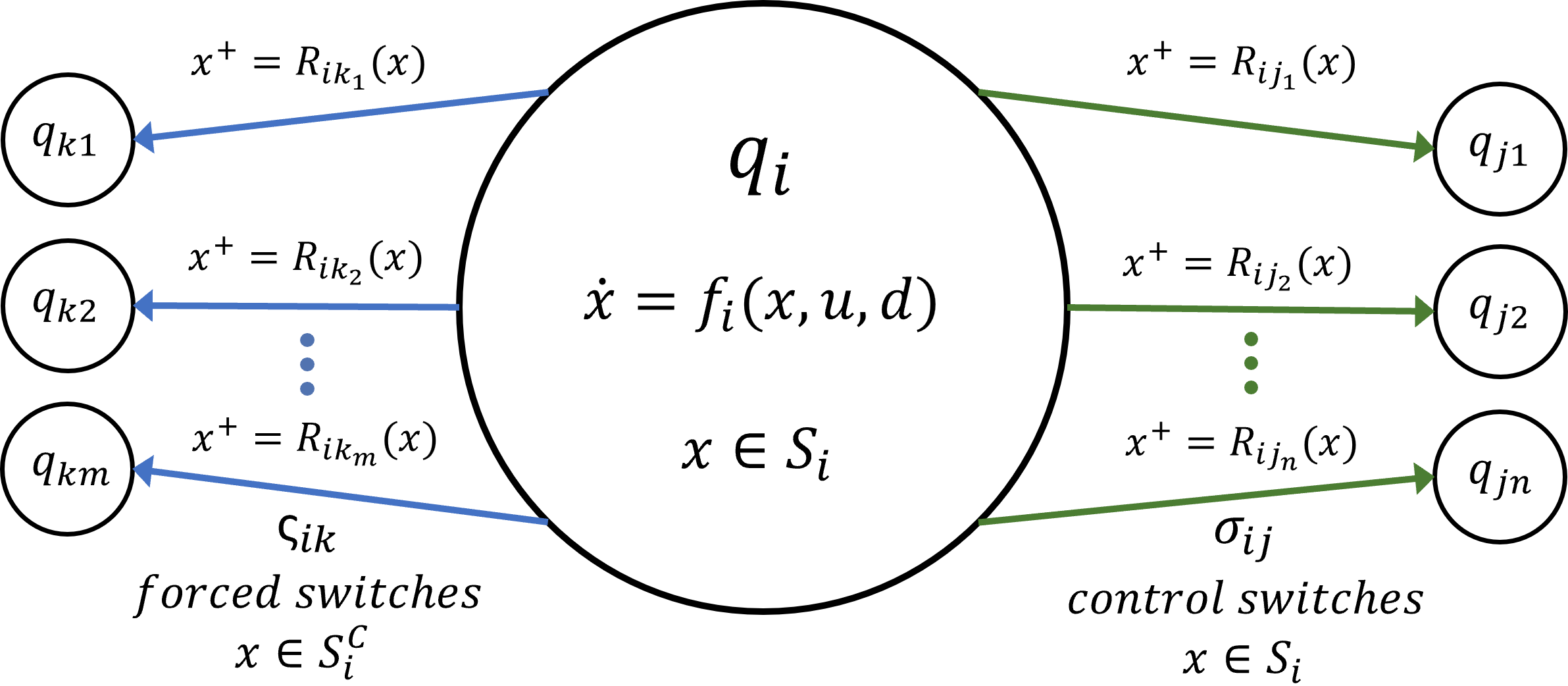}
\captionof{figure}{Hybrid dynamical system with controlled and forced transitions.}\label{fig:dyn_diagram} 
\vspace{-1.0em}
\end{center}
\end{figure}
In mode $q_i$, the system has discrete control switches $\sigma_{ij} \in \Sigma$ that allow a \textit{controlled transition} into another discrete mode $q_{j}$, where $j \in \{j_1,j_2, \hdots,j_n\} \subset \{1, 2, \hdots, N\}$.
The controlled transition can occur only when $x \in S_i$.
Whenever the evolution of the continuous state makes $x$ exit the valid operation domain for mode $q_i$ ($x \in {S_i}^C$), the system must take one of the \textit{forced control switches}, $\varsigma_{ik} \in \Sigma$.
This leads to a \textit{forced transition} into another discrete mode $q_{k}$ 
where $k \in \{k_1,k_2,...,k_m\} \subset \{1, 2, \hdots, N\}$. 
Note that the number of discrete modes the system can transition into may vary across different $q_i$.
Each discrete transition from $q_i$, whether controlled or forced, might lead to a state reset, which is given by the reset relation $R$ in \eqref{eq:hyb_def}.
For simplicity, we denote the state reset map upon transitioning from mode $q_i$ to $q_j$ as
$x^+ = R_{ij}(x)$, where $x$ is the state before the transition and $x^+$ is the state after the transition.
This hybrid dynamical model is represented in Figure~\ref{fig:dyn_diagram}.

\subsection{\label{problem_brt}Backward Reachable Tubes for Hybrid Systems}

Hybrid Backward Reachable Tubes (hBRTs) can be used to reason about both \textit{safety} and \textit{liveness} in hybrid systems. When used for safety, the hBRT identifies the set of unsafe states from which the system is guaranteed to enter a failure set within a finite time horizon, regardless of the hybrid control policy. Conversely, when used for liveness, the hBRT defines the live set of states from which the system can be steered into a desired goal despite worst-case disturbances. In this work, when mathematically addressing the hBRT, we adopt the liveness interpretation: we define the hBRT as the set of initial continuous and discrete states $(x, q)$ from which the system is guaranteed to be able to reach a goal set $\mathcal{G}$ within the time horizon $[t, T]$.

We adopt this perspective to simplify the exposition of results that follow by aligning objectives under a common value-based framework. Specifically, both the liveness hBRT and hBRAT characterize desirable outcomes (e.g., reaching a goal or avoiding failure) as subzero level sets of a value function, allowing for a unified treatment of controller synthesis through minimization of this function. When relevant, we include remarks on the corresponding safety hBRT formulation.

To mathematically define the hBRT, we follow the notation in \cite{mitchell2005time}.
Let $u(\cdot)$ and $d(\cdot)$ denote the control and disturbance functions over time. 
We further assume that these functions $u(\cdot), d(\cdot)$ are drawn from the set of functions $\mathcal{U}(t)$ and $\mathcal{D}(t)$ correspondingly.
Since the control and disturbance are playing a differential game against each other, it is important to address what information they know about each other's decisions. 

We also assume that the disturbance only uses nonanticipative strategies $\gamma \in \Gamma(t)$, where
$\Gamma(t)$ denotes the set of all nonanticipative mappings from control signals to disturbance signals, that is strategies:
\begin{align}
& \gamma \in \Gamma(t) \triangleq\{\vartheta: \mathcal{U}(t) \rightarrow \mathcal{D}(t) \mid u(r)=\hat{u}(r) \nonumber\\
& \text { for almost every } r \in[t, s] \\
& \Longrightarrow \vartheta[u](r)=\vartheta[\hat{u}](r) \text { for almost every } r \in[t, s]\} .\nonumber
\end{align}
Informally, this restriction means that if disturbance cannot distinguish between control signals $u(\cdot)$ and $\hat{u}(\cdot)$ until after time $s$, then it cannot respond differently using an arbitrary mapping $\vartheta: \mathcal{U}(t) \rightarrow \mathcal{D}(t)$ to those signals until after time $s$. Under this setting, the disturbance still has an instantaneous advantage, as it can respond after observing the current control action. 
It turns out that under nonanticipative strategies, the BRT can be obtained using a value function that is the solution of certain Hamilton-Jacobi-Isaacs PDE \citep{mitchell2005time}. 

We are now ready to formally define the liveness hBRT of the system in \eqref{eq:hyb_def}:
{\fontsize{8.8}{10}\selectfont
\begin{align}
&hBRT(t)=\{(x,q_i):\ 
 \forall d(\cdot)\in\{\gamma[u](\cdot)\mid \gamma\in\Gamma(t)\},\  \\ 
&\exists u(\cdot)\in\mathcal{U}(t),
 \sigma(\cdot),\varsigma(\cdot)\in\Sigma(t),\ 
\exists s\in[t,T]\ \text{s.t. }\ 
\zeta^{u,d,\sigma,\varsigma}_{x,q_i,t}(s)\in\mathcal{G}
\}.\nonumber
\end{align}
}
where $\zeta_{x,q_i, t}^{u, d, \sigma,\varsigma}(s)$ denotes the trajectory of the system at time $s$, starting from state $x$ and discrete mode $q_i$ at the initial time $t$, under continuous control profile $u(\cdot)$, discrete control switch profile $\sigma(\cdot)$, forced control switch profile $\varsigma(\cdot)$, and a non-anticipative disturbance strategy $d(\cdot)$.
Along with the hBRT, we are also interested in obtaining the continuous and discrete control laws, $u^*(\cdot)$, $\sigma^*(\cdot)$ and $\varsigma^*(\cdot)$, that optimally steer the system towards the goal set.

Finally, if we are interested in computing the safety hBRT, the goal set $\mathcal{G}$ should be replaced with a set of failure states for the system (such as obstacles for a navigation robot), and the role of controls and disturbance switched, thus finding the set of initial states from which getting into the failure set is unavoidable, despite best control effort.

\subsection{\label{problem_fltr}Safety Filtering for Hybrid Systems}

While the computation of the hBRT provides a rigorous characterization of the unsafe region (i.e., the set of states from which failure is inevitable) it is not sufficient on its own. In practice, safety must be maintained while the system pursues operational goals. To address this, we consider the integration of safety guarantees into the execution of a nominal policy $\pi_{nom}$, which is typically designed to optimize performance metrics like efficiency, speed, or task success, but often without formal consideration of the system’s safety constraints.

Our first objective is to synthesize a \textbf{safety filter policy} $\pi_{safe}$ such that, when used alongside $\pi_{nom}$, the hybrid system $H$ is guaranteed to avoid the failure set $\mathcal{L}$ for all times within the time horizon $[t, T]$, while modifying the nominal behavior as little as possible. Achieving this requires a formal method for determining when safety is at risk and how to intervene in a minimally invasive way. To this end, we make use of $hBRT(t)$, introduced in the previous subsection~\ref{problem_brt}, which provides the foundation for reasoning about safety in hybrid systems.

In the context of safety filtering, we assume that the nominal hybrid control policy first selects and executes any discrete transitions based on the current hybrid state and time, i.e., $\sigma_{nom}(x, q_i, t) = \sigma_{ij}$ and $\varsigma_{nom}(x, q_i, t) = \varsigma_{ik}$. This includes the use of the dummy control transition $\sigma_{ii}$ that keeps the system in the same mode when no switch is taken. These transitions result in a particular intermediate updated state $(\tilde{x}, \tilde{q}_i)$, upon which the continuous control input is then evaluated as $u_{nom}(\tilde{x}, \tilde{q}_i, t)$. This allows the safety filter to assess the effect of discrete transitions before continuous time evolution occurs.

\subsection{\label{problem_RA}Backward Reach-Avoid Tubes for Hybrid Systems}

While the previous objective ensures that the system can operate safely under an arbitrary nominal policy, it does not guarantee that the system will achieve its intended goals. In many scenarios, safety must be enforced without compromising the ability to complete a task. Therefore, we extend our formulation to avoid unsafe regions and ensure that the system reaches a desired goal set within a finite time horizon, resulting in formal guarantees of both safety and task completion.

This requirement defines our second objective, which is to compute the \textbf{Hybrid Backward Reach-Avoid Tube (hBRAT)} of this hybrid dynamical system, defined as the set of initial discrete and continuous states $(x,q_i)$ of the system such that starting from these states, for all disturbance inputs, there exists a control input that will steer the system to a goal set $\mathcal{G}$ within the time horizon $[t, T]$ without ever entering a failure set $\mathcal{L}$ before.
Using the notation defined in subsection~\ref{problem_brt}, we formally define the hBRAT of the hybrid system in \eqref{eq:hyb_def} as:
\vspace{-0.5em}

{\fontsize{8.8}{10}\selectfont
\begin{align}
{hBRAT}(t)=\{(x, q_i): \forall d(\cdot)\in\{\gamma[u](\cdot)\mid \gamma\in\Gamma(t)\},\nonumber\\
\exists u(\cdot) \in \mathcal{U}(t),
\sigma(\cdot);\varsigma(\cdot) \in \Sigma(t),\nonumber\\
\exists s \in [t, T], \zeta_{x,q_i,t}^{u, d,\sigma,\varsigma}(s) \in \mathcal{G}~\wedge 
\forall r \in [t, s]  ~\zeta_{x,q_i,t}^{u, d, \sigma,\varsigma}(r) \notin \mathcal{L} \}.
\end{align}
}
The hBRAT captures the ability to simultaneously guarantee safety and goal satisfaction, and will serve as the basis for synthesizing hybrid control policies that guarantee achieving task objectives while avoiding unsafe states. The construction and numerical solution of this set will be described in detail in later sections.

%% file: 04a_background.tex
\section{\label{background} Background}
\subsection{\label{background_hj}Hamilton-Jacobi Reachability}

One way to compute the BRT or the more general Backward Reach-Avoid Tubes (BRAT) for continuous-time dynamical systems is through Hamilton-Jacobi reachability analysis. 
Within our framework this can be considered as a system with only one discrete mode and no discrete transitions.

The BRAT computation is formulated as a zero-sum game between control and disturbance. This results in a robust optimal control problem that can be solved using the dynamic programming principle. 
First, a goal function $g(x)$ and failure function $l(x)$ are defined such that their sub-zero and super-zero level sets corresponds to the goal set $\mathcal{G}$ and failure set $\mathcal{L}$, i.e. $\mathcal{G} = \{x : g(x)\leq 0\}$ and $\mathcal{L} = \{x : l(x)\geq 0\}$. The BRAT seeks to find all states that could enter $\mathcal{G}$ within the time horizon without ever getting into $\mathcal{L}$ during that time.
This is captured by the cost function:
\begin{multline} 
J(x, t, u(\cdot), d(\cdot))=\\\min _{s \in[t, T]} \max 
\{g(\zeta_{x, t}^{u, d}(s)), \max _{r \in[t, s]} l(\zeta_{x, t}^{u, d}(r))\},
\end{multline}
where $\zeta_{x, t}^{u, d}(s)$ denotes the continuous trajectory of the system at time $s$, starting from state $x$ at initial time $t$, under continuous control profile $u(\cdot)$, and a non-anticipative disturbance strategy $d(\cdot)$.

The objective is to find the minimum cost for optimal system trajectories. Thus, it considers the optimal control that minimizes this cost (drives the system towards the goal) and the worst-case disturbance signal that maximizes the cost (drives the system towards the failure set and/or away from the goal). The value function corresponding to this robust optimal control problem is:
\begin{equation}\label{eq:hji}
 V(x, t)=\adjustlimits\max _{d \in D} \min _{u \in U} \{J(x, t, u(\cdot), d(\cdot))\}.
\end{equation}
The value function in (\ref{eq:hji}) can be computed using dynamic programming, which results in the following final value Hamilton-Jacobi-Isaacs Variational Inequality (HJI-VI) \citep{Fisac15,herbert2020safe}:
\begin{equation}\label{eq:hji_vi}
\begin{array}{c}
\max [
\min \{D_{t} V(x, t)+\mathcal{H}(x, t), l(x)-V(x, t)\},\\ g(x)-V(x, t)]=0\\ \\
V(x, T)=\max[l(x),g(x)].
\end{array}
\end{equation}
$D_t$ and $\nabla$ represent the time and spatial gradients of the value function. $\mathcal{H}$ is the Hamiltonian, which optimizes over the inner product between the spatial gradients of the value function and the dynamics to compute the optimal control and disturbance:
\begin{equation} \label{eqn:ham}
\mathcal{H}(x, t)=\max _{d \in D} \min _{u \in U}\nabla V(x, t) \cdot f(x, u, d).
\end{equation}
For a detailed derivation and discussion of the HJI-VI, we refer the interested readers to \cite{mitchell2005time,herbert2020safe} and \cite{bansal2017hamilton}. 

Once the value function is obtained, the BRAT is given as the sub-zero level set of the value function:
\begin{equation}
BRAT(t)=\{x: V(x, t) \leq 0\}.
\end{equation}
The corresponding optimal control can be derived as:
\begin{equation}\label{eq:optctrl}
u^{*}(x, t)=\argmin_{u \in U} \max_{d \in D}\nabla V(x, t) \cdot f(x, u, d).
\end{equation}
The system can guarantee reaching the goal set without entering the failure set as long as it starts inside the BRAT and applies the optimal control in (\ref{eq:optctrl}) at the BRAT boundary. The optimal disturbance can be similarly obtained.

Multiple computation tools exist to compute the value function and obtain the BRAT and the optimal controller.
This include methods that solve the HJI-VI in \eqref{eq:hji} numerically \citep{bansal2020provably, bui2022optimizeddp,mitchell2004toolbox,hj_reach_ASL2023} or using learning-based methods \citep{bansal2021deepreach,fisac2019bridging}.
However, one of the key limitations of classical HJ reachability analysis is assuming continuous dynamics and not accounting for discrete mode switching or state resets.

%% file: 04b_prior_work.tex
\subsection{\label{prior_work} HJ Reachability for Hybrid Systems}
This subsection covers the core contribution of our preliminary paper \citep{borquez2024hybridreach}, which centers around Theorem~\ref{theorem1}, an extension of the classical HJ reachability framework for the calculation of hBRTs of dynamical systems with controlled and forced transitions, as well as state resets.
\begin{theorem}\label{theorem1}
Consider the hybrid dynamical system $H$ as defined in \eqref{eq:hyb_def}.
Let $g(x)$ be an implicit representation of the goal set $\mathcal{G}$, i.e. $\mathcal{G} = \{x : g(x)\leq 0\}$.
Also let the value function $V(x, q_i, t)$ be the solution of the following constrained HJI-VI:\vspace{-0.5em}
\begin{subequations} \label{eq:hybrid_hji_brt}
\begin{align}
\text{If}~ x &\in S_i,  \nonumber \\
\min\{ &\min_{\sigma_{ij}}V(R_{ij}(x), q_{j}, t) - V(x, q_i, t), \nonumber \\ 
&D_{t}V(x, q_i, t)+ \max_{d \in D} \min_{u \in U} \nabla V(x, q_i, t)\cdot f_i(x,u,d), \nonumber \\
&g(x)-V(x, q_i, t)\} = 0. \label{eq:hybrid_hji_brt_a} \\[3pt]
\text{If}~x &\notin S_i, \nonumber \\
&\min_{\varsigma_{ik}}\{V(R_{ik}(x), q_{k}, t) - V(x, q_i, t)\} = 0 .\label{eq:hybrid_hji_brt_b}
\end{align}
\end{subequations}
with terminal time condition:
\vspace{-0.4em}
\begin{equation}
V(x, q_i, T) = g(x). 
\end{equation}
Then, the liveness hBRT for the system is given as:
\begin{equation}
hBRT(t)=\{(x, q_i): V(x, q_i, t) \leq 0\}.
\end{equation}
\end{theorem}

Intuitively, Theorem~\ref{theorem1} updates the value function for a mode $q_i$ by selecting the best outcome among all discrete transitions available from that mode. As illustrated in Figure~\ref{fig:opt_decision}, the outer minimization in~\eqref{eq:hybrid_hji_brt_a} chooses either (i) the best discrete switch $\sigma_{ij}$ with its associated reset map $R_{ij}(x)$, reflected in the first term, or (ii) remaining in the same mode $q_i$, which yields the continuous evolution of the value function represented by the second and third terms corresponding to the classical HJ reachability formulation.
To formally account for the possibility of staying in the current mode, we introduce a ``dummy'' discrete control $\sigma_{ii}$ that keeps the system in mode $q_i$. 
\begin{figure}[h!] 
\begin{center} 
\vspace{-0.5em}
\includegraphics[width=0.95\columnwidth]{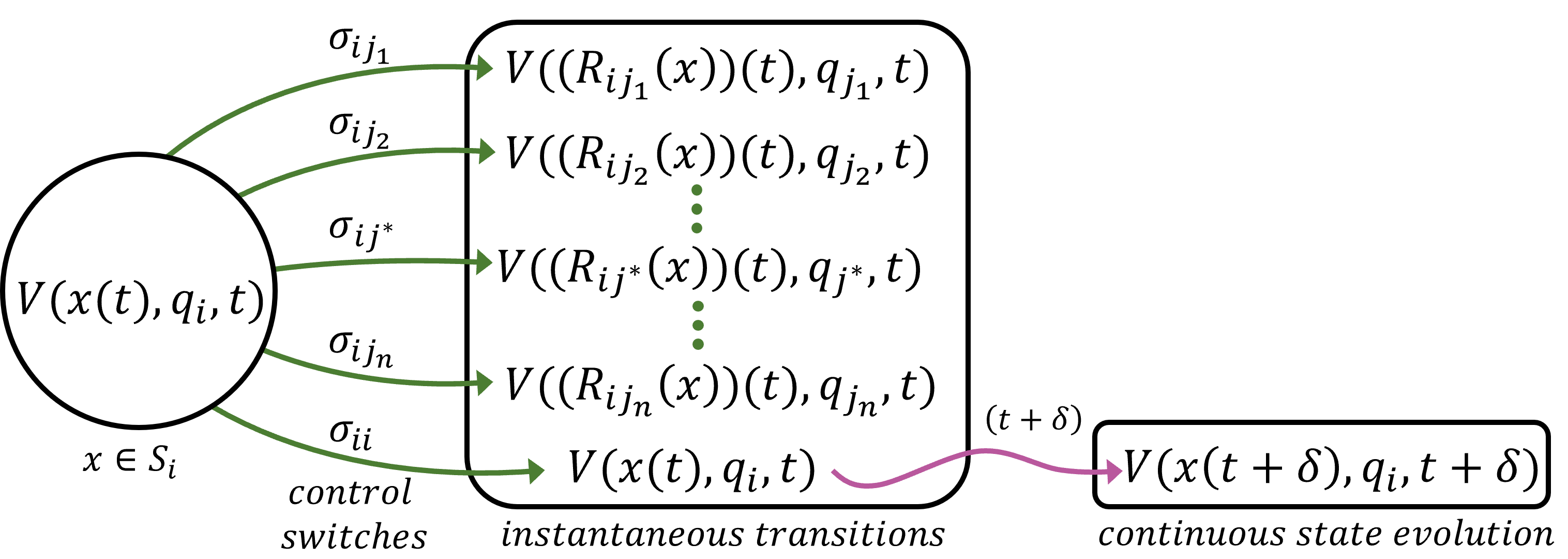}
\caption{The value function is updated by selecting the best among (i) discrete transitions to other modes $q_j$ via control switches or (ii) leveraging a ``dummy'' transition to stay in the current mode and update the value based on the continuous evolution of the system.} \label{fig:opt_decision} 
\vspace{-1em}
\end{center}
\end{figure}

For states outside the operation domain $S_i$ (i.e., in $S_i^C$), the update in~\eqref{eq:hybrid_hji_brt_b} has the same structure, except the option to remain in $q_i$ is removed. The system must take a forced switch $\varsigma_{ik}$, each with its reset map $R_{ik}(x)$, and the value function is simply the minimum over these mandatory transitions. The proof of Theorem~\ref{theorem1} is provided in our preliminary work~\cite{borquez2024hybridreach}.

Once we compute the value function, the optimal discrete and continuous controls at state $(x, q_i)$ at time $t$ are given as:
\begin{align}\label{eqn:opt_control_hybrid}
u^{*}(x, q_i, t)&=\argmin_{u \in U} \max_{d \in D}\nabla V(x, q_i, t) \cdot f_i(x, u, d), \nonumber\\
\sigma_{ij^*}(x, q_i, t) &=\argmin_{\sigma_{ij}}V(R_{ij}(x), q_{j}, t) \text{~if~} x \in S_i,\\
\varsigma_{ik^*}(x, q_i, t) &=\argmin_{\varsigma_{ik}}V(R_{ik}(x), q_{k}, t) \text{~if~} x \in S_i^C.\nonumber
\end{align}
\begin{remark}
    Note that the proposed framework can easily be extended to scenarios where the forced transition mode represent an adversarial or uncertain transition instead. 
    In this case, we can use $\max$ instead of $\min$ over $\varsigma_{ik}$ in Theorem 1 to account for the worst-case behavior.
\end{remark}

\begin{remark}
    This formulation can also be applied to compute the avoid hBRT by replacing the goal set with the failure set and reversing the roles of control and disturbance in Theorem~\ref{theorem1}. Additionally, the discrete switching decision in (\ref{eqn:opt_control_hybrid}) changes from a minimization to a maximization, reflecting the controller's objective to avoid the failure set rather than reach a goal.
\end{remark}

\noindent \textbf{Numerical implementation:} We now present an approximate numerical algorithm that can be used to calculate the value function in Theorem 1. 
It builds upon the value function calculation for the classical HJI-VI in \eqref{eq:hji}, which is solved using currently available level set methods \citep{mitchell2004toolbox}. 
Specifically, the value function is computed over a discretized state-space grid $\hat{x}$, and propagated in time with timestep $\delta$ for each mode $q_i$.
After each propagation step, the value function over $\hat{x}$ is updated for each mode $q_i$ according to the discrete transitions available at each grid point as per~\eqref{eq:hybrid_hji_brt}, this is repeated until the time horizon $T$ is reached.
The detailed procedure is presented in Algorithm \ref{alg1}. 
\begin{algorithm}[t]
  \caption{Value function computation for hybrid dynamical system}\label{alg1} 
 \textbf{Input:} $g(x), T, \delta, \hat{x},H$\\
 \textbf{Output:} $V(\hat{x},q_i,t) \forall i$\\
  \textbf{Initialization:} $V(\hat{x},q_i,T)\gets g(\hat{x}) \forall i$; \, $t\gets T$\\
\While{$(t > 0)$}{
\ForEach{$(q_i \in Q)$}{
\begingroup
\footnotesize
$V(\hat{x},q_i,t-\delta)\gets V(\hat{x},q_i,t) +
    \newline \hspace*{5em}\displaystyle \max_{d} \min_{u}\nabla V(\hat{x}, q_i, t)\cdot f_i(\hat{x}, u, d) \delta$\\
\endgroup
}
\ForEach{$(q_i \in Q)$}{
\uIf{$(\hat{x} \in S_i$)}{
\begingroup
\footnotesize
    $\displaystyle V(\hat{x}, q_i, t-\delta) \gets \min\{ g(\hat{x}), V(\hat{x}, q_i, t-\delta),
    \newline \hspace*{6.8em} \min_{\sigma_{ij}}V(R_{ij}(\hat{x}), q_{j}, t-\delta) \}$
\endgroup
  }
  \uElseIf{$(\hat{x} \in S_i^C$))}{
\begingroup
\footnotesize
     $\displaystyle V(\hat{x}, q_i, t-\delta) \gets 
 \min_{\varsigma_{ik}}V(R_{ik}(\hat{x}), q_{k}, t-\delta) $
\endgroup
  }
}
$t=t-\delta$
}
\end{algorithm}

It is important to stress the remarkable similarity of this new hybrid reachability algorithm to its continuous counterpart in \eqref{eq:hji}. 
There are only two key differences: (1) we now propagate $N$ value functions simultaneously (corresponding to $N$ discrete modes), as opposed to just one.
(2) We ``adjust'' the value function for each discrete mode to account for discrete switches (second for-loop in Algorithm \ref{alg1}).
Despite this simplicity, the algorithm simultaneously reasons about all possible discrete and continuous transitions to optimally steer the system to the goal set.

\textbf{\textit{Complexity analysis}}. This algorithm computes the hBRT over a state-space grid, using a finite time step.  
Assuming there are $M$ grid points per state dimension, the total is $M^{n_x}$ grid points, where $n_x$ is the number of continuous states. 
At each timestep, for each discrete mode, the computation of the value function requires evaluating all potential discrete transitions from the current mode (the inner for loop in Algorithm \ref{alg1}). This results in $\mathcal{O}(N \cdot M^{n_x})$ computations per timestep, per mode. Since there are $N$ discrete modes, the computational requirement per timestep scales to  $\mathcal{O}(N^2 \cdot M^{n_x})$. Finally, the total number of timesteps is proportional to the time horizon $T/\delta$, resulting in a total complexity of $\mathcal{O}(T \cdot N^2 \cdot M^{n_x} / \delta)$. 
In summary, the computational complexity of Algorithm 1 scales linearly with the time horizon, quadratically with the number of discrete modes, and exponentially with the number of continuous state dimensions.

\begin{remark}\label{remark:hbrt_aprox}
Note that Algorithm \ref{alg1} computes an approximate value function. As with all grid-based HJ methods, exact computation would require infinitely fine spatial and temporal discretization. Beyond this overarching limitation, determining exact hybrid reachable sets is, in general, undecidable \citep{lygeros1998controller, henzinger1995s}. Consequently, Algorithm \ref{alg1} may not terminate in finite time for arbitrary hybrid dynamics. Despite these theoretical limitations, the proposed approach has proven highly effective in practice for approximating hBRTs, as demonstrated in several case studies and real-robot experiments in our preliminary work~\cite{borquez2024hybridreach}.
\end{remark}

%% file: 05a_theory_LR.tex
\section{\label{proof_LR}Safety Filtering for Hybrid Systems}

In this section, we address the hybrid safety filtering problem introduced in Subsection~\ref{problem_fltr}, where the goal is to enforce safety in a hybrid system by modifying a nominal control policy only when necessary to prevent eventual entry into the failure set. To this end, we propose and validate a minimally intrusive hybrid safety control filter that wraps around a given nominal hybrid control policy and ensures that the system remains within the safe set throughout its operation.

For a hybrid dynamical system $H$ as described in \eqref{eq:hyb_def}, we consider a general nominal hybrid control policy:  

\vspace{-1em}
{\fontsize{8.8}{10}\selectfont
\begin{align}  
\pi_{nom}(x, q_i, t) &= (u_{nom}(x, q_i, t), \sigma_{nom}(x, q_i, t), \varsigma_{nom}(x, q_i, t)) \notag \\  
&= (u_{nom}(\tilde{x}, \tilde{q}_i, t), \sigma_{ij}, \varsigma_{ik}).  \label{eq:hyb_nom}
\end{align}
}
To render such a nominal policy provably safe, we define a safety-guaranteeing \textbf{Hybrid Least Restrictive Filter (hLRF)}  $\pi_{safe}=\pi_{lr}=(u_{lr},\sigma_{lr},\varsigma_{lr})$ as:
\begingroup
\small
\begin{equation}\label{eq:hyb_lr}
\begin{aligned}
&\text{If}~x \in S_i:\\
&\begin{aligned}
& \sigma_{lr}= \begin{cases}\sigma_{nom} (x, q_i, t)=\sigma_{ij}&\hspace{0em}   V(R_{ij}(x), q_{j})\geq0,\\
\sigma_{i j^*}(x, q_i) &\hspace{0em} V(R_{ij}(x), q_{j})<0.\end{cases}\\
&\text{If}~x \notin S_i:\\
& \varsigma_{lr}= \begin{cases}\varsigma_{nom} (x, q_i, t)=\varsigma_{ik} &\hspace{0em} V(R_{ik}(x), q_{k})\geq0, \\
\varsigma_{i k^*}(x, q_i) &\hspace{0em}V(R_{ik}(x), q_{k})<0.\end{cases}\\
\\
&\text{After the transition}: \\
& u_{lr}= \begin{cases}u_{nom}(\tilde{x}, \tilde{q_i}, t) & V\left(\tilde{x}, \tilde{q_i}\right)>0, \\
u^*\left(\tilde{x}, \tilde{q_i}\right) & V(\tilde{x}, \tilde{q_i})=0.\end{cases}
\end{aligned}
\end{aligned}
\end{equation}
\endgroup

\noindent where we have access to the time-converged safety hBRT\footnote[2]{The filter can be analogously defined for time-dependent safety hBRTs, reach hBRTs, and hBRATs; we focus on the time-converged avoid case as all-time safety is the primary use case of reachability-based filtering.}  
 of the system with its associated value function $V(x,q_i)$ given by \eqref{eq:hybrid_hji_brt}, and a hybrid optimal safety policy $\pi^{*}(x, q_i)=(u^*(x, q_i),\sigma_{i j^*}(x, q_i),\varsigma_{i k^*}(x, q_i))$ given by \eqref{eqn:opt_control_hybrid}.
 
In simpler terms, the hLRF acts like a safety supervisor that constantly monitors the system’s planned actions, both discrete transitions and continuous control inputs, and only steps in if the planned action would lead to a safety violation. For discrete transitions, the filter checks whether the currently planned transition is safe; if so, it allows it. If the transition would push the system into unsafe territory, the filter overrides it with the safest alternative. After the transition, the filter similarly monitors the continuous control input and only adjusts it if the system is on the boundary of safety, ensuring the system stays inside the safe set. This way, the filter minimally disrupts the original plan while guaranteeing safety at all times. This safety guarantees provided by the hLRF are formalized in the following theorem.%

\begin{theorem}\label{theorem_hlrf}
Consider the hybrid dynamical system $H$ as defined in \eqref{eq:hyb_def}, and $V(x, q_i)$ the time-converged value function encoding the avoid hBRT of a failure set $\mathcal{L}$. Then for an arbitrary hybrid nominal policy $\pi_{nom}$ as defined in \eqref{eq:hyb_nom} the use of the hLRF $\pi_{lr}$ as defined in \eqref{eq:hyb_lr} ensures forward invariance of the safe set.
\end{theorem}

\textit{Proof:} We aim to show that if the system starts from a safe state $(x, q_i)$ such that $V(x, q_i) \geq 0$, then the next state, after applying the hybrid safety filtering strategy, is also safe.

\vspace{0.5em}
\noindent \textbf{Step 1: Discrete Transition Safety.} We first prove that the discrete transition selected by the filter maintains safety; that is, $V(\tilde{x}, \tilde{q}_i) \geq 0$ holds under all four possible cases.

\begin{itemize}
    \item \textbf{Case a1} ($x \in S_i$, $V(R_{ij}(x), q_j) \geq 0$):  
    The nominal controlled switch $\sigma_{ij}$ leads to a safe state. Since $\tilde{x} = R_{ij}(x)$ and $\tilde{q}_i = q_j$, we have $V(\tilde{x}, \tilde{q}_i) = V(R_{ij}(x), q_j) \geq 0$ directly.

    \item \textbf{Case a2} ($x \in S_i$, $V(R_{ij}(x), q_j) < 0$):  
    The nominal controlled switch leads to an unsafe state, so the filter overrides it with the optimal safe switch $\sigma_{ij^*}$. By definition of the value function, we have:
    \[
    V(x, q_i) = \max_{\sigma_{ij}} V(R_{ij}(x), q_j) = V(R_{ij^*}(x), q_{j^*}).
    \]
    Since we assumed $V(x, q_i) \geq 0$, it follows that $V(\tilde{x}, \tilde{q}_i) = V(R_{ij^*}(x), q_{j^*}) \geq 0$. This case also includes the ``dummy'' control $\sigma_{ii}$, which keeps the system in the same mode if it is safer to remain there.

    \item \textbf{Case b1} ($x \notin S_i$, $V(R_{ik}(x), q_k) \geq 0$):  
    The nominal forced switch $\varsigma_{ik}$ leads to a safe state. Thus, $V(\tilde{x}, \tilde{q}_i) = V(R_{ik}(x), q_k) \geq 0$.

    \item \textbf{Case b2} ($x \notin S_i$, $V(R_{ik}(x), q_k) < 0$):  
    The nominal forced transition is unsafe and is overridden with the optimal forced transition $\varsigma_{ik^*}$. Similarly, we have:
    \[
    V(x, q_i) = \max_{\varsigma_{ik}} V(R_{ik}(x), q_k) = V(R_{ik^*}(x), q_{k^*}).
    \]
    Again, using the assumption that $V(x, q_i) \geq 0$, we conclude that $V(\tilde{x}, \tilde{q}_i) = V(R_{ik^*}(x), q_{k^*}) \geq 0$.
\end{itemize}

\noindent In all cases, the discrete update ensures $V(\tilde{x}, \tilde{q}_i) \geq 0$.

\vspace{0.5em}
\noindent \textbf{Step 2: Continuous Control Safety.}  
After the discrete transition, the system operates in a fixed mode $\tilde{q}_i$ with continuous dynamics. The continuous safety filter $u_{lr}$ applied here corresponds to the \textit{Continuous-Time Least Restrictive Safety Filter}, proven to maintain safety in \cite{borquezFiltering2023}. Since this filter guarantees that trajectories remain within the safe set defined by the BRT, the state remains safe after the application of continuous control, completing the proof. \hfill $\square$

%% file: 05b_cases_LR.tex
\subsection{\label{casesLR}Hybrid Least Restrictive Filtering Case Studies}
To demonstrate the effectiveness of the proposed hLRF approach, we apply it to two distinct hybrid systems. These case studies illustrate how the safety filter guarantees safety while minimally modifying a nominal policy. Each system presents unique challenges, such as forced transitions, controlled transitions, unactuated modes, and timing constraints, highlighting the flexibility of the proposed method in different hybrid control scenarios.

\subsubsection{\label{jump4D_LR}Planar Jumping Robot:}
To illustrate our approach, we consider a simple planar jumping robot with the hybrid dynamics shown in Fig.~\ref{fig:jump4d_diagram}:
\vspace{-1.0em}
\begin{figure}[h!]\centering
     \includegraphics[width=0.65\columnwidth]{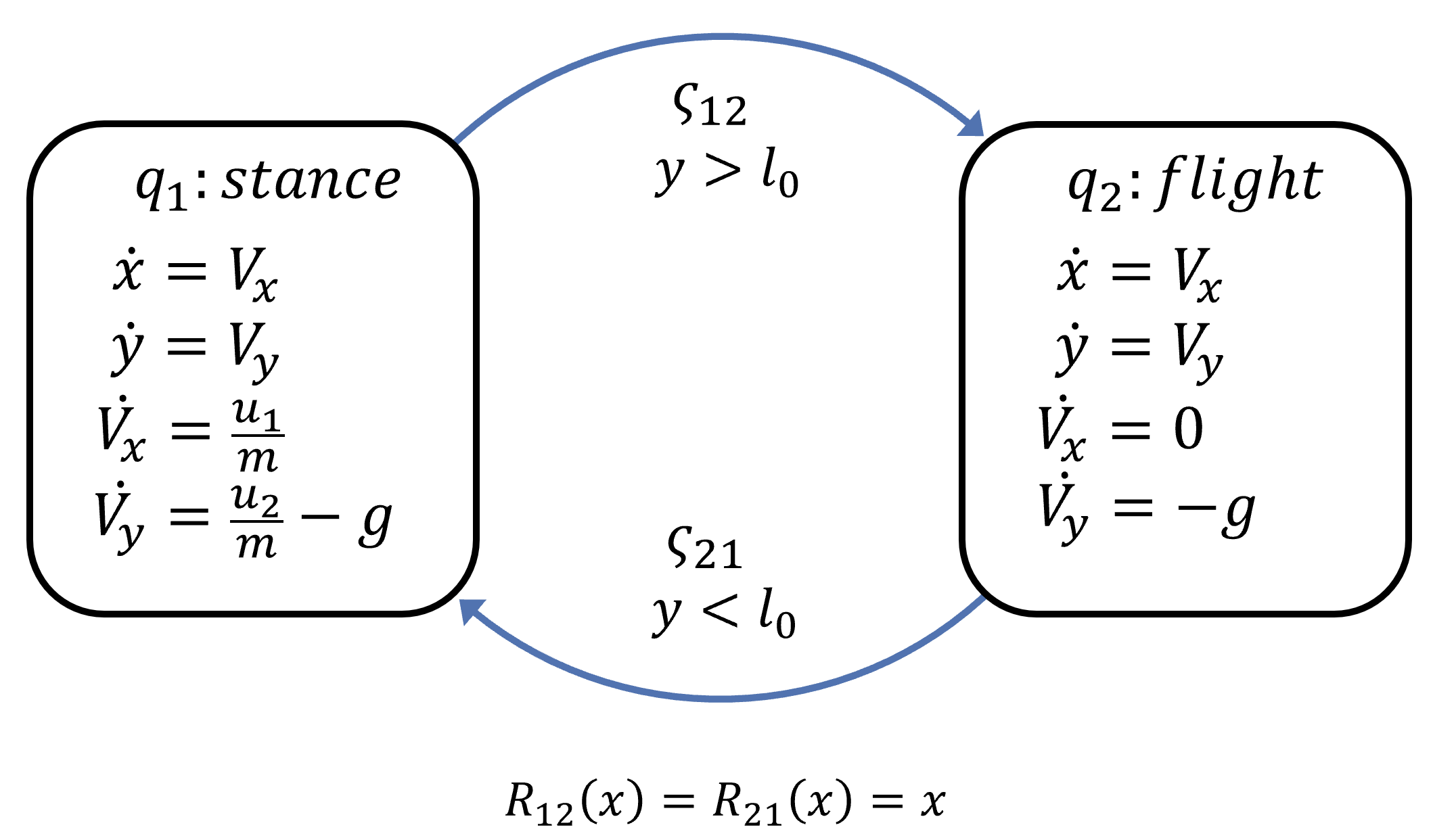}
     \vspace{-0.5em}
     \caption{Hybrid dynamics of the planar jumping robot with two modes: stance ($q_1$), where ground-actuated forces apply, and flight ($q_2$), with ballistic motion. Forced transitions are triggered when the vertical position reaches the leg-length boundary.}
      \label{fig:jump4d_diagram} 
\end{figure}
 
Here, the continuous state $x \in \mathbb{R}^4$ represents the $xy$ positions and velocities of the center of the robot, and $q_i$ with $i \in \{1,2\}$ the discrete modes that model the behavior of the system while standing on the ground and when in ballistic flight after jumping. 
$(u_1,u_2)$ are the control force inputs in each axis with an actuation range of $0-30$~N, this force is applied against the ground so it is only available while in stance mode. The transition is modeled by a forced transition that checks if the robot $y$ distance to the ground is larger or smaller than the robot's leg length $l_0=0.5$~m. We also consider $m=1$~kg the mass of the robot, and $g$ is the acceleration due to gravity. 

We consider the failure set shown in red in Fig.~\ref{fig:jump4d_avoid}, where the obstacles are taller than the robot’s leg length and must be cleared by jumping. To evaluate our safety filter, we simulate two representative nominal control strategies that do not account for obstacle locations or safety constraints.\footnote{We specify only the continuous control component of each policy, as the system involves only forced transitions with a single transition candidate, and the flight mode ($q_2$) is unactuated. Thus, the discrete and flight-mode nominal policies are irrelevant for this example.}

The first policy, $u_{nom1} = (5, g)$, applies horizontal force while compensating for gravity, producing straight-line ground movement. The second, $u_{nom2} = (5, 15)$, applies additional vertical thrust, resulting in low ballistic hops that cannot clear the obstacles.

We render both policies safe by computing the time-converged avoid hBRT of the failure set and using the resulting value function $V(x, q_i)$ to construct the hLRF described in \eqref{eq:hyb_lr}.

\begin{figure}[t] 
\includegraphics[width=0.95\columnwidth]{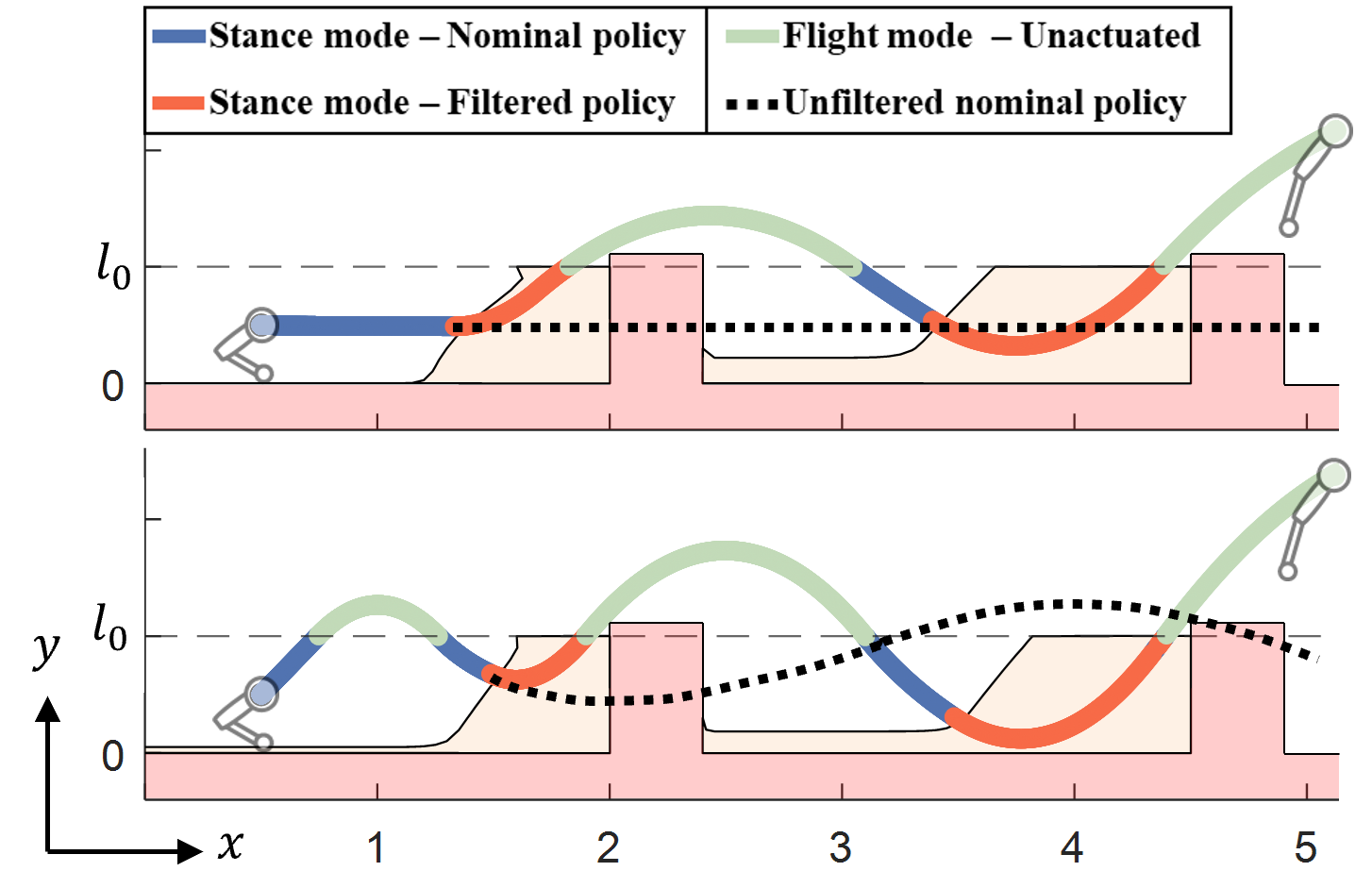}
      \captionof{figure}{hLRF renders nominal policies safe in a jumping robot example. Top plot corresponds to $u_{nom1}$ (horizontal movement), bottom plot shows $u_{nom2}$ (hopping behavior). In light orange, we show slices of the unsafe hBRT at the states that trigger the filter by coming into contact with its boundary. Dashed lines show the unsafe unfiltered trajectories.}
      \label{fig:jump4d_avoid} 
\end{figure}

Filtered trajectories for both nominal policies, starting from rest in the stance mode, are shown in Fig.~\ref{fig:jump4d_avoid}. The required hybrid avoid BRT is computed over a grid with $0.05$~m spatial and $0.2$~m/s velocity resolution. We visualize slices of the hBRT in light orange at the exact states where the hLRF activates, a BRT slice represents a lower-dimensional cross-section of the reachable tube, obtained by fixing certain state variables, in this particular case we show the xy cross section of the tube at the particular velocities of the first state at risk of breaching safety. Trajectories are color-coded: blue indicates nominal stance behavior, orange shows filtered stance behavior, and light green corresponds to unactuated flight states. Dashed black lines show the unsafe, unfiltered trajectories, which would lead the system into failure.

The top plot corresponds to the horizontal nominal policy $\pi_{nom1}$. While in stance mode, the safety filter triggers when a collision is imminent, generating a corrective jump to safely clear the obstacle. These interventions are shown in orange and guide the system into flight mode with safe velocity and position. For the second obstacle, the filter engages earlier due to the incoming negative velocity from the first jump, as reflected in the larger hBRT slice at that contact point.

In the bottom plot, the nominal policy $\pi_{nom2}$ produces low hops that would otherwise fail to clear the obstacle. The hLRF intervenes just in time, adjusting the trajectory to avoid failure. Notably, the ability to query the value function at all states ensures safety is maintained regardless of the nominal behavior, as the filter can always determine the optimal safe action to stay outside the hBRT at its boundary. This generality and reactivity come with minimal computational cost as the hLRF executes in real time, averaging just $0.16$~ms per query on an Intel i7-13620H processor.

\subsubsection{\label{Sat5D_LR}Low orbit spacecraft:}

We further validate the hLRF on a spacecraft model adapted from \cite{Satellite_CBF}, using nonlinear planar Hill's dynamics~\citep{curtis2019orbital}. Here $p$, $r$, $V_p$, $V_r$ denote along-track and radial displacements and velocities relative to a chief satellite, $s$ a timer, and $\alpha$, $\beta$ state-dependent accelerations encoding Coriolis and gravitational effects, with $\mu$ Earth's gravitational parameter, and $\omega$, $R$ the chief's orbital angular velocity and radius. The hybrid dynamics are shown in Fig.~\ref{fig:sat_diag}.

\vspace{-0.5em}
\begin{figure}[h]\centering
\includegraphics[width=1.0\columnwidth]{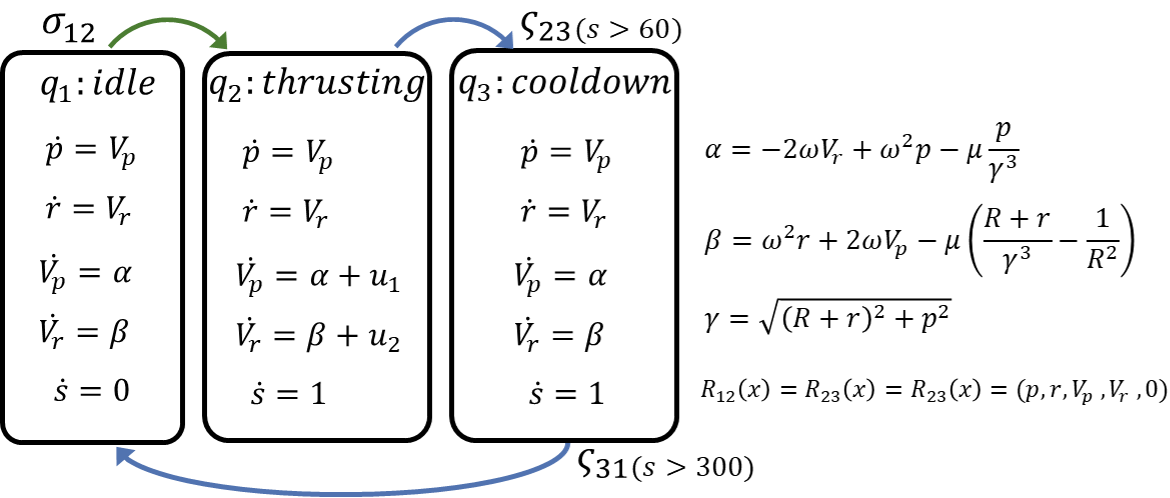}
    \vspace{-2.0em}
      \captionof{figure}{Low orbit spacecraft hybrid system formulation.}
      \label{fig:sat_diag} 
\end{figure}
\vspace{-2em}
\begin{itemize}
    \item \textbf{Mode \( q_1 \) (Idle):} The spacecraft keeps its thrusters off, awaiting a controlled transition to engage propulsion.
    \item \textbf{Mode \( q_2 \) (Thrusting):} The spacecraft uses its thrusters for 60~s. After this period, a forced transition takes the system into cooldown mode.
    \item \textbf{Mode \( q_3 \) (Cooldown):} The spacecraft is unactuated and unable to re-engage its thrusters. After a 300-second cooldown, it transitions back to idle mode.
\end{itemize}

\vspace{-0.5em}
We consider the spacecraft operating near a chief satellite surrounded by obstacles as shown in Fig.~\ref{fig:sat_traj}. The deputy must navigate a cluttered environment following the presented hybrid dynamics and continuous control inputs $u_1$ and $u_2$ bounded by $\pm1$~$\mathrm{m/s^2}$. As nominal hybrid policy we consider a passive ``do-nothing'' strategy that applies no thrust and triggers no discrete transitions.

The proposed hLRF supervises this nominal policy and overrides it when necessary to prevent collisions. Filtered trajectories from two initial conditions appear in Fig.~\ref{fig:sat_traj}. In all cases, the safety filter maintains collision-free motion with real-time performance, averaging $0.19$~ms per query on the previously described hardware.

Color-coded trajectories indicate the active discrete mode: idle (blue), thrusting (orange), and cooldown (gray). Safety is enforced by handling both controlled and forced transitions, respecting timing constraints, and managing continuous controls where propulsion is available during the value-function computation.

\vspace{-0.25em}
\begin{figure}[b]\centering
\includegraphics[width=0.755\columnwidth]{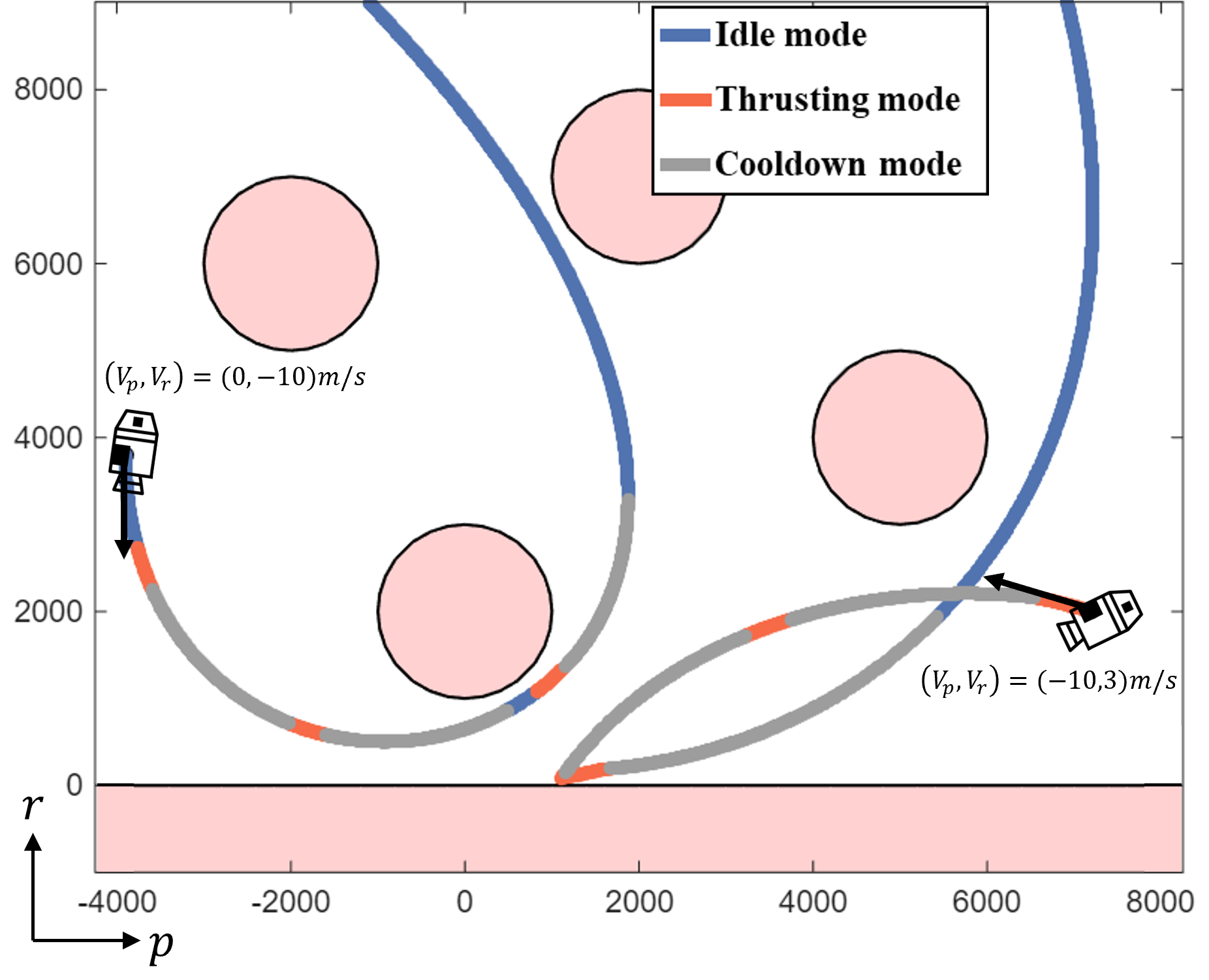}
        \vspace{-1.0em}
      \captionof{figure}{Low orbit spacecraft trajectories where the hLRF filters a nominal do-nothing policy to avoid collisions.}
      \label{fig:sat_traj} 
\end{figure}

%% file: 06a_theory_BRAT.tex
\section{\label{hybrid_BRAT}Reach-Avoid Formulation for Hybrid Systems}
In previous sections, we introduced the hLRF as a mechanism to ensure safety by filtering unsafe actions in real time. While this provides strong safety guarantees, many applications also require task completion guarantees, such as reaching a goal location within specific time constraints. In this section, we address this complementary need by introducing a reach-avoid formulation for hybrid systems that guarantees both safety and task completion.

We present Theorem~\ref{theorem2}, an extension of Theorem~\ref{theorem1}, to tackle the hybrid reach-avoid problem discussed in Subsection~\ref{problem_RA}. The objective is to compute the set of initial states from which the hybrid system can reach a goal set within a given time horizon while avoiding a failure set. While Theorem~\ref{theorem1} provides a framework for computing the hBRT for either reachability or avoidance in isolation, Theorem~\ref{theorem2} generalizes this approach to enforce both objectives simultaneously.

Intuitively, Theorem~\ref{theorem2} mirrors the structure of Theorem~\ref{theorem1}, with the key difference that the continuous evolution term now follows the BRAT variational HJI-VI in~\eqref{eq:hji_vi}. In~\eqref{eq:hybrid_hji_brat_a}, the outer minimization again selects the best discrete decision available from mode \(q_i\): either (i) taking a discrete transition \(\sigma_{ij}\) and applying the corresponding reset map \(R_{ij}(x)\), or (ii) remaining in mode \(q_i\), in which case the value evolves according to~\eqref{eq:hji_vi}. For states in \(S_i^C\), the update in~\eqref{eq:hybrid_hji_brat_b} removes this option, forcing the system to apply a transition \(\varsigma_{ik}\) and take the minimum over all mandatory resets.

\begin{theorem}\label{theorem2}
Consider the hybrid dynamical system $H$ as defined in \eqref{eq:hyb_def}.
Let $g(x)$ be an implicit representation of the goal set $\mathcal{G}$, i.e. $\mathcal{G} = \{x : g(x)\leq 0\}$, and $l(x)$ an implicit representation of the failure set $\mathcal{L}$, i.e. $\mathcal{L} = \{x : l(x)\geq 0\}$.
Also let the value function $V(x, q_i, t)$ be the solution of the following constrained HJI-VI:
\vspace{-0.5em}
\begingroup
\small
\begin{subequations} \label{eq:hybrid_hji_brat}
\begin{align}
&\text{If } x \in S_i, \nonumber \\
&\min {\Big\{} \min_{\sigma_{ij}}V(R_{ij}(x), q_{j}, t)-V(x, q_i, t), \nonumber\\
&\quad\quad\quad \max {\Big[} l(x)-V(x, q_i, t),\label{eq:hybrid_hji_brat_a} \\
&\quad\quad\quad\quad\quad\quad \min{\Big(} g(x)-V(x, q_i, t), \nonumber \\
&D_{t}V(x, q_i, t) +
 \max_{d \in D} \min_{u \in U} \nabla V(x, q_i, t)\cdot f_i(x,u,d) {\Big)}{\Big]}{\Big\}} = 0.  \nonumber
  \\[3pt]
&\text{If } x \notin S_i, \nonumber \\
&\quad\quad \min_{\varsigma_{ik}}\{V(R_{ik}(x), q_{k}, t) - V(x, q_i, t)\} = 0.
\label{eq:hybrid_hji_brat_b}
\end{align}
\end{subequations}
\endgroup
with terminal time condition:\vspace{-0.5em}
\begin{equation}
V(x, q_i, T) = \max(g(x),l(x)). 
\end{equation}
Then, the hBRAT for the hybrid system is given as:\vspace{-0.5em}
\begin{equation}
\mathcal{V}(t)=\{(x, q): V(x, q, t) \leq 0\}.
\end{equation}
\end{theorem}
%

\textit{Proof:} We start with an analogous formulation to HJ reachability with a goal function $g(x)$ such that the goal set is defined as its sub-zero level set, $\mathcal{G} = \{(x) : g(x)\leq 0\}$, as well as a failure function $l(x)$ such that the failure set is defined as its super-zero level set, $\mathcal{L} = \{(x) : l(x)\geq 0\}$. 
The hBRAT seeks to find all states that could enter $\mathcal{G}$ at any point within the time horizon without ever getting into $\mathcal{L}$.

To simplify notation we redefine the trajectory notation introduced in subsection \ref{problem_brt} as $\zeta(s) := \zeta_{x,q_i, t}^{u, d, \sigma,\varsigma}(s)$. 
The cost is defined as the minimum distance to the goal set over time after taking the maximum between this distance and the maximum clearance to the failure set.
\vspace{-0.5em}
\begingroup
\small
\begin{multline} 
J(x, q_i, u(\cdot), d(\cdot), \sigma(\cdot),\varsigma(\cdot), t)=\\\min_{s \in[t, T]} \max \{g(\zeta(s)), \max _{r \in[t, s]} l(\zeta(r))\}.
\end{multline}
\endgroup

\noindent \textbf{Scenario A (Interior of $S_i$):} We begin by analyzing the scenario where the state lies in the interior of the domain $S_i$. In this setting, continuous evolution and controlled discrete transitions are admissible.

The value function is defined as the cost under optimal continuous and discrete controls that minimize it, and optimal disturbances that maximize it, the value is given as:
\begingroup
\small
\begin{equation}\label{eq:hyb_value_def}
V(x, q_i, t) = \max_{d \in D} \min_{u \in U} \min_{\sigma \in \Sigma}J(x, q_i, u(\cdot), d(\cdot), \sigma(\cdot), t).
\end{equation}
\endgroup
With terminal time condition given by:
\vspace{-0.2em}
\begingroup
\small
\begin{equation}
V(x, q_i, T) = \max [g(x),l(x)]. 
\end{equation}
\endgroup

The discrete transitions considered by the system are shown in green arrows and the continuous control evolution is represented by purple trajectories in Fig.~\ref{fig:opt_decision}. We once again consider a ``dummy'' discrete transition, $\sigma_{ii}$, which allows the system to remain in its current discrete mode. As defined in \eqref{eq:hyb_value_def}, the overall value function is given by the minimum cost over all possible trajectories. Based on this, we proceed by analyzing two distinct cases:

\noindent \textbf{Case A1:} If the system switches to a new discrete mode $j^*$, the continuous state immediately transitions to $R_{ij^*}(x)$. 
In this case, the optimal cost incurred by the system from the new state $(R_{ij^*}(x), q_{j^*})$ is, by definition, given as $V(R_{ij^*}(x), q_{j^*}, t)$. 
Thus, the optimal cost across all possible discrete transitions is given as:
\vspace{-0.2em}
\begingroup
\small
\begin{equation}
V(x, q_i, t) =\min_{\sigma_{ij}}V(R_{ij}(x), q_{j}, t).
\end{equation}
\endgroup

\noindent \textbf{Case A2:} The system evolves in the same mode $q_i$, i.e., it takes the dummy switch $\sigma_{ii}$. 
In this case, for a small time step $\delta$ and using the dynamic programming principle for the cost function in \eqref{eq:hyb_value_def} we can write the value function as\footnote[3]{Algebraic steps are omitted for brevity's sake; for a step-by-step derivation for the non-hybrid case, we refer the reader to Chapter 2 in \cite{herbert2020safe}.}:
%
\definecolor{color1}{RGB}{200,50,50}
\definecolor{color2}{RGB}{40,180,40}
\definecolor{color3}{RGB}{50,50,200}
\begingroup
\small
\begin{align} \label{eqn:cont_flow}
&V(x, q_i, t) = \\
&\max_{d \in D} \min_{u \in U}  \textcolor{color1}{\min\Big(} \min_{s \in[t, t+\delta]} \textcolor{color3}{\max\Big(} g(\zeta(s)),  \max_{r \in[t, s]} l(\zeta(r)) \textcolor{color3}{\Big)}, \notag\\
 & \textcolor{color3}{\max\Big(} V(x(t+\delta), q_i, t+\delta), \max_{r \in[t, t+\delta]} l(\zeta(r)) 
 \textcolor{color3}{\Big)} \textcolor{color1}{\Big )}.\notag
\end{align}
\endgroup
Considering that the time step $\delta$ is infinitely small and approximating the value function at the next time step with a first order Taylor expansion:
\begingroup
\small
\begin{align}
&V(x, q_i, t) = 
\max_{d \in D} \min_{u \in U} \textcolor{color1}{\min\Big(} \textcolor{color3}{\max\Big(}g(\zeta(t)),l(\zeta(t))\textcolor{color3}{\Big)} , \\
& \textcolor{color3}{\max\Big(} V(x, q_i, t) \hspace{-0.2em}+\hspace{-0.3em} D_{t}V(x, q_i, t)\delta \hspace{-0.2em}+\hspace{-0.3em}\nabla V(x, q_i, t)\cdot\delta x , l(\zeta(t))\textcolor{color3}{\Big)}\textcolor{color1}{\Big)}. \notag
\end{align}
\endgroup
Reorganizing the min and max terms, and noting that a trajectory evaluated at time $t$ reduces to the state at that time, we write $\zeta(t) = x(t) := x$, omitting explicit time dependence of the state for notational simplicity:

\begingroup
\small
\begin{align}
V(x, q_i, t) =& 
\max_{d \in D} \min_{u \in U} \textcolor{color3}{\max\Big(} l(x),\textcolor{color1}{\min\Big(}g(x),\\
&V(x, q_i, t) +  D_{t}V(x, q_i, t)\delta +\nabla V(x, q_i, t)\cdot\delta x \textcolor{color1}{\Big)}\textcolor{color3}{\Big)}.\notag
\end{align}
\endgroup

The change in the state $\delta x$ can be approximated as $f_i(x,u,d)\delta$, and since only this term depends on the control and disturbance, the value function becomes:
\begingroup
\small
\begin{align}
V(x, q_i,& t) = 
\textcolor{color3}{\max\Big(} l(x),\textcolor{color1}{\min\Big(}g(x),V(x, q_i, t) + ...\\
& D_{t}V(x, q_i, t)\delta + \max_{d \in D} \min_{u \in U}\nabla V(x, q_i, t)\cdot  f_i(x,u,d) \delta \textcolor{color1}{\Big)}\textcolor{color3}{\Big)}.\notag
\end{align}
\endgroup

\noindent \textbf{Combining Cases:} The value function is obtained by taking the pointwise minimum over the two cases, reflecting the optimal decision at each state.
\begingroup
\small
\begin{align}
&V(x, q_i, t) = \textcolor{color2}{\min\Big(} \min_{\sigma_{ij}}V(R_{ij}(x), q_{j}, t),\\
&\textcolor{color3}{\max \Big(} l(x),\textcolor{color1}{\min\Big(}g(x), V(x, q_i, t) \hspace{-0.2em}+\hspace{-0.3em}  D_{t}V(x, q_i, t)\delta \hspace{-0.2em}+\hspace{-0.3em}...\notag\\
&\max_{d \in D} \min_{u \in U}\nabla V(x, q_i, t)\cdot  f_i(x,u,d) \delta \textcolor{color1}{\Big)}\textcolor{color3}{\Big)}\textcolor{color2}{\Big)}.\notag
\end{align}
\endgroup
Subtracting the value function $V(x, q_i, t)$ from both sides:
\begingroup
\small
\begin{align}
0 = &\textcolor{color2}{\min\Big(} \min_{\sigma_{ij}}V(R_{ij}(x), q_{j}, t)-V(x, q_i, t),\\
&\textcolor{color3}{ \max \Big(} l(x)-V(x, q_i, t),\textcolor{color1}{\min\Big(}g(x)-V(x, q_i, t),\notag\\
&D_{t}V(x, q_i, t)\delta  + \max_{d \in D} \min_{u \in U}\nabla V(x, q_i, t)\cdot  f_i(x,u,d) \delta \textcolor{color1}{\Big)}\textcolor{color3}{\Big)}\textcolor{color2}{\Big)}.\notag
\end{align}
\endgroup
Since the above statement holds for all $\delta>0$, we must have: 
\begingroup
\small
\begin{align}
0 = &\textcolor{color2}{\min\Big(} \min_{\sigma_{ij}}V(R_{ij}(x), q_{j}, t)-V(x, q_i, t),\\
&\textcolor{color3}{ \max \Big(} l(x)-V(x, q_i, t),\textcolor{color1}{\min\Big(}g(x)-V(x, q_i, t),\notag\\
&D_{t}V(x, q_i, t)  + \max_{d} \min_{u} \nabla V(x, q_i, t) f_i(x,u,d) \textcolor{color1}{\Big)}\textcolor{color3}{\Big)}\textcolor{color2}{\Big)}.\notag
\end{align}
\endgroup

Note that the procedure presented to derive the HJI-VI for \textit{continuous evolution} of the value function is an informal proof based on the Taylor series expansion that assumes the value function to be differentiable, which may not be true, and only the existence of a viscosity solution can be ensured \citep{mitchell2005time, lygeros2004reachability}.
Nevertheless, this HJI-VI is known to hold even when the value function is non-differentiable, formal proof can be found in \cite{lygeros2004reachability}.
We omitted a detailed derivation for brevity purposes.

\vspace{0.4em}
\noindent \textbf{Scenario B (Outside $S_i$):} For states outside the operational domain of mode $q_i$ (i.e., $x \in S_i^C$), the system must take a forced discrete transition, selecting from the available switches $\varsigma_{ik}$. The structure of the value function update mirrors case A1, as there is no equivalent to case A2 as in this scenario the system is not allowed to stay in $q_i$. The value function is then given by:
\vspace{-0.4em}
\begin{equation}
V(x, q_i, t) =\min_{\varsigma_{ik}}V(R_{ik}(x), q_{k}, t).
\end{equation}

\noindent completing the proof. \hfill $\square$
\vspace{1em}

After computing the value function associated with the hBRAT, the optimal continuous and discrete controls at state $(x, q_i)$ and time $t$ are obtained using the same expressions as in the hBRT case:

\begin{align}\label{eqn:opt_control_hybrid_brat}
u^{*}(x, q_i, t)&=\argmin_{u \in U} \max_{d \in D}\nabla V(x, q_i, t) \cdot f_i(x, u, d), \nonumber\\
\sigma_{ij^*}(x, q_i, t) &=\argmin_{\sigma_{ij}}V(R_{ij}(x), q_{j}, t) \text{~if~} x \in S_i,\\
\varsigma_{ik^*}(x, q_i, t) &=\argmin_{\varsigma_{ik}}V(R_{ik}(x), q_{k}, t) \text{~if~} x \in S_i^C.\nonumber
\end{align}


\noindent \textbf{Numerical implementation:} We present an approximate numerical algorithm to compute the value function associated with the hBRAT in Theorem~\ref{theorem2}. It closely parallels the hBRT computation described in Algorithm~\ref{alg1}. As before, the value function is propagated backward in time over a discretized grid $\hat{x}$ using a small timestep $\delta$. At each step, it is updated according to the continuous dynamics and then adjusted to account for discrete transitions across all modes.

The key differences are twofold: (1) the terminal condition is initialized differently, using a pointwise maximum of the goal and failure functions to encode that the system must reach the goal set while avoiding the failure set. (2) The update step now follows the structure of \eqref{eq:hybrid_hji_brat}. Which for states within the operation domain involves evaluating a more complex structure combining goal satisfaction, failure avoidance, continuous evolution, and possible discrete transitions. For clarity, we separate this into two steps in Algorithm~\ref{alg2}, though conceptually, it mirrors the single-step update used for the hBRT in Algorithm~\ref{alg1}.

\begin{algorithm}[b]
  \caption{Value function computation for hybrid Reach-Avoid problem}\label{alg2} 
 \textbf{Input:} $l(x),g(x), T, \delta, \hat{x},H$\\
 \textbf{Output:} $V(\hat{x},q_i,t) \forall i$\\
  \textbf{Initialization:}
  \begingroup
    \footnotesize
  $V(\hat{x}, q_i, T) \gets \max [l(\hat{x}),g(\hat{x})]  \forall i$; \,  $t\gets T$\\
    \endgroup  
\While{$(t > 0)$}{
\ForEach{$(q_i \in Q)$}{
\begingroup
\footnotesize 
$V(\hat{x},q_i,t-\delta)\gets V(\hat{x},q_i,t) + \newline \hspace*{7em}
    \displaystyle \max_{d} \min_{u}\nabla V(\hat{x}, q_i, t)\cdot f_i(\hat{x}, u, d) \delta$
\endgroup
}
\ForEach{$(q_i \in Q)$}{
\uIf{$(\hat{x} \in S_i$)}{
    \begingroup
    \footnotesize 
    $V(\hat{x},q_i,t-\delta)\gets \max\{l(\hat{x}),\newline \hspace*{8.05em}\min\{ g(\hat{x}), V(\hat{x}, q_i, t-\delta)\}\}$\\
    $\displaystyle V(\hat{x}, q_i, t-\delta) \gets\min\{V(\hat{x}, q_i, t-\delta),\newline \hspace*{9.5em}\min_{\sigma_{ij}}V(R_{ij}(\hat{x}), q_{j}, t-\delta)\}$
    \endgroup
  }
  \uElseIf{$(\hat{x} \in S_i^C$)}{
    \begingroup
    \footnotesize 
    $\displaystyle V(\hat{x}, q_i, t-\delta) \gets  \min_{\varsigma_{ik}}V(R_{ik}(\hat{x}), q_{k}, t-\delta) $
    \endgroup
  }
}
\begingroup
\footnotesize 
$t\gets t-\delta$
\endgroup
}
\end{algorithm}

\textit{\textbf{Complexity analysis:}}
The computational cost of Algorithm~\ref{alg2} matches that of the hBRT algorithm (Algorithm~\ref{alg1}). Although the hBRAT update inside the operational domain has a slightly richer structure, the algorithm still performs the same per–grid point operations and iterates over all modes and states at each time step. Thus, the scaling remains unchanged, and the overall complexity is 
$\mathcal{O}(T \cdot N^2 \cdot M^{n_x} / \delta)$, 
with $N$ discrete modes, $M$ grid points per continuous dimension, $n_x$ the number of continuous states, time horizon $T$, and discretization step $\delta$.

\begin{remark}
Limitations discussed in Remark~\ref{remark:hbrt_aprox} also apply to the hBRAT computation presented here. As with Algorithm~\ref{alg1}, Algorithm~\ref{alg2} relies on discretized approximations and may not yield exact results without infinite resolution.
\end{remark}

%% file: 06b_cases_BRAT.tex
\subsection{\label{casesRA}Hybrid Reach-Avoid Case Studies}

We validate the proposed hBRAT framework through two hybrid system case studies designed to test its ability to synthesize guaranteed goal-reaching behavior under safety constraints. These examples showcase how the approach ensures that the system reaches a specified goal set within a finite time horizon while avoiding entry into defined failure regions. Each scenario introduces distinct hybrid control challenges—including mode transitions, unactuated dynamics, and dwell-time constraints—demonstrating the versatility and robustness of the hybrid reach-avoid formulation across different system configurations.

\subsubsection{\label{jump4dRA}Planar Jumping Robot}

To illustrate our hybrid reach avoid approach, we consider the planar jumping robot example introduced in subsection~\ref{jump4D_LR}. We will use the hybrid dynamical system formulation shown in Fig.~\ref{fig:jump4d_diagram} and the same simulation parameters previously discussed.

In this case, we add a yellow circular goal area after the obstacles shown in red, as depicted in Fig.~\ref{fig:jump4d_RA}.
The hBRAT for this example corresponds to all the combinations of continuous states and discrete modes $(x,q_i)$ from which the robot can reach the goal set within a $2$ second time horizon without colliding with the failure set.

A trajectory following the optimal reach-avoid policy is shown in Fig.~\ref{fig:jump4d_RA}. As we are considering an initial state contained within the hBRAT, we are guaranteed to safely reach the goal set within the time horizon by following the optimal control policy supported by the calculated value function. The color-coded trajectory shows the stance states in blue and the unactuated flight states in light green. Slices of the hBRAT before and after the transitions from stance to flight modes are shown in blue and green correspondingly. As with the safety filter, querying the value function for policy execution is efficient—averaging just $0.16$~ms per query on an Intel i7-13620H processor—enabling real-time application of the reach-avoid policy.

The results in Fig.~\ref{fig:jump4d_RA} confirm that the robot successfully reaches the goal without colliding with the failure set within the time horizon ($1.8$~s for this example), as guaranteed by the hybrid reach-avoid framework. Notably, the plot highlights the transition between the hBRAT slices corresponding to the stance and flight modes, which come in contact precisely at the boundary of the forced jump transition. This alignment indicates that the reach-avoid set for both discrete modes is consistent across transitions, ensuring a seamless and provably safe mode-switching behavior throughout the trajectory.

\begin{figure}[b] 
\includegraphics[width=0.95\columnwidth]{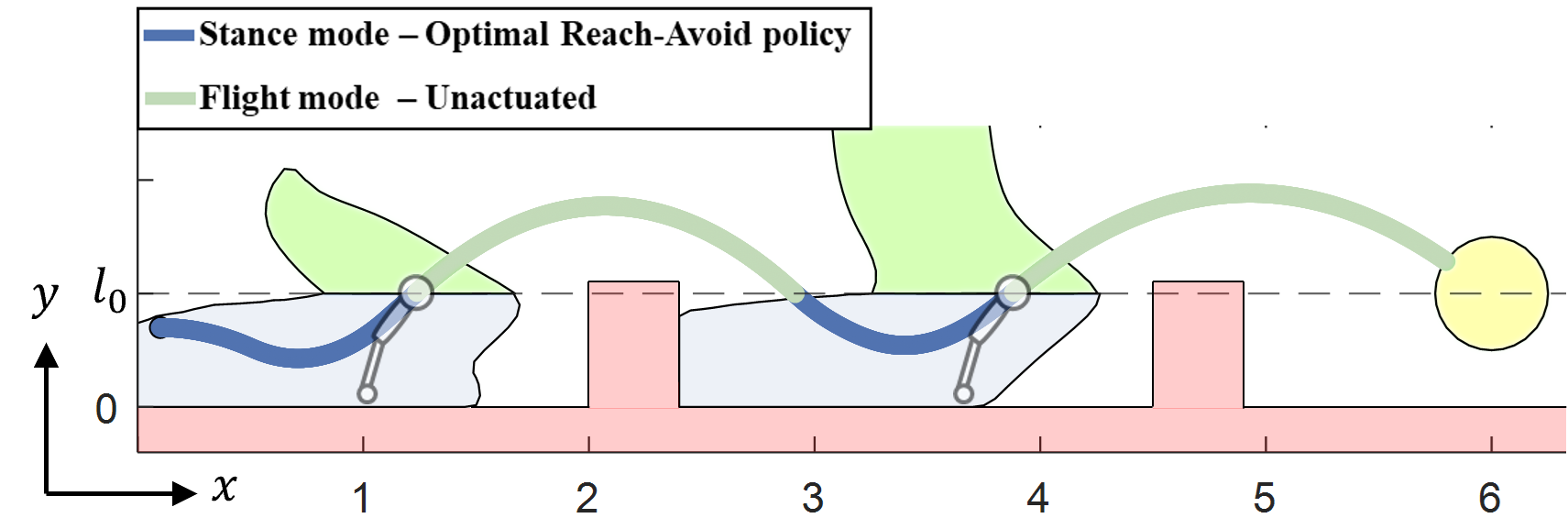}
      \captionof{figure}{Hybrid reach-avoid trajectory for the planar jumping robot example. The robot safely reaches the goal (yellow) without contacting obstacles (red), as ensured by the computed hBRAT. The hBRAT slices for stance and flight modes align at the forced jump transition, ensuring seamless reach-avoid guarantees even in the presence of mode switching.}
      \label{fig:jump4d_RA} 
\end{figure}


\subsubsection{\label{Sat5D_RA}Low Earth Orbit Spacecraft}

To further evaluate the hBRAT approach, we apply it to the low Earth orbit spacecraft example introduced in subsection~\ref{Sat5D_LR}. The system follows the same hybrid dynamical formulation shown in Fig.~\ref{fig:sat_diag}, with simulation parameters consistent with the safety filtering case study, with the addition of a reach time horizon $T=780$~s.

In this scenario, instead of just avoiding obstacles, the goal is to also guide the spacecraft toward a successful docking with the chief satellite located at the origin. To achieve this, we use the proposed approach to compute the hBRAT, which characterizes the set of initial states and modes from which the spacecraft can safely reach the docking point within a finite time horizon while avoiding obstacles.

Figure~\ref{fig:sat_traj_RA} shows optimal reach-avoid trajectories for two idle initial conditions. Corresponding hBRAT slices are overlaid for each case, highlighting the regions from which the spacecraft can safely reach the docking point within the time horizon. First, we note how the hBRAT aligns with intuitive expectations. In the left plot, where the spacecraft starts with low velocity, the hBRAT slice is compact and situated close to the goal, with almost no clearance around obstacles in the direction of motion. In contrast, the right plot considers a high radial velocity initial condition, leading to an elongated hBRAT slice along this direction, reflecting the greater distance the spacecraft can cover per unit of time. However, this also results in larger obstacle margins in the same direction, as the actuation is insufficient to avoid collisions if it starts too close.

Building on these observations, the trajectories in Fig.~\ref{fig:sat_traj_RA} demonstrate how the reach-avoid controller successfully guides the spacecraft to the docking point while adhering to the hBRAT constraints. In the left plot, where the spacecraft starts with low velocity, the optimal policy executes a gradual approach, smoothly adjusting the trajectory to avoid obstacles while maintaining a controlled descent toward the docking point. In contrast, the right plot, where the spacecraft begins with high radial velocity, showcases a more aggresive trajectory. The controller strategically utilizes the initial momentum to cover more ground efficiently before engaging in a single thrust mode transition to direct itself toward the docking area safely. These trajectories highlight the reach-avoid controller's ability to handle a variety of initial conditions optimally even for nonlinear hybrid dynamics.

\begin{figure}[b]\centering
\includegraphics[width=0.99\columnwidth]{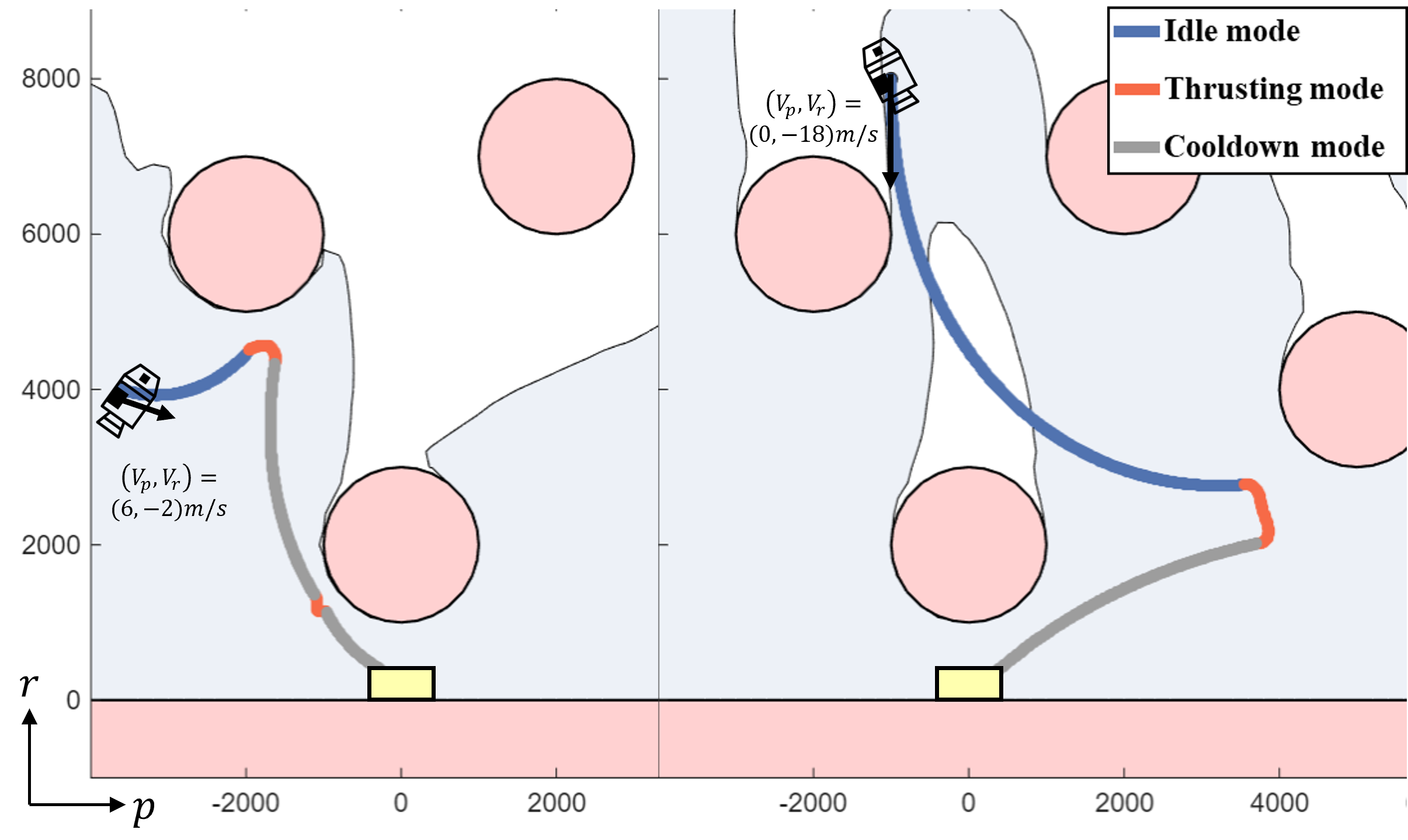}
      \captionof{figure}{Reach-avoid trajectories for the low Earth orbit spacecraft. The spacecraft successfully reaches the goal set while avoiding obstacles, as ensured by the computed hBRAT. The hBRAT slices for each initial condition is shown in light blue.}
      \label{fig:sat_traj_RA} 
\end{figure}

%% file: 10_hw_exp.tex
\begin{figure*}[t]
\centering
\includegraphics[width=\textwidth]{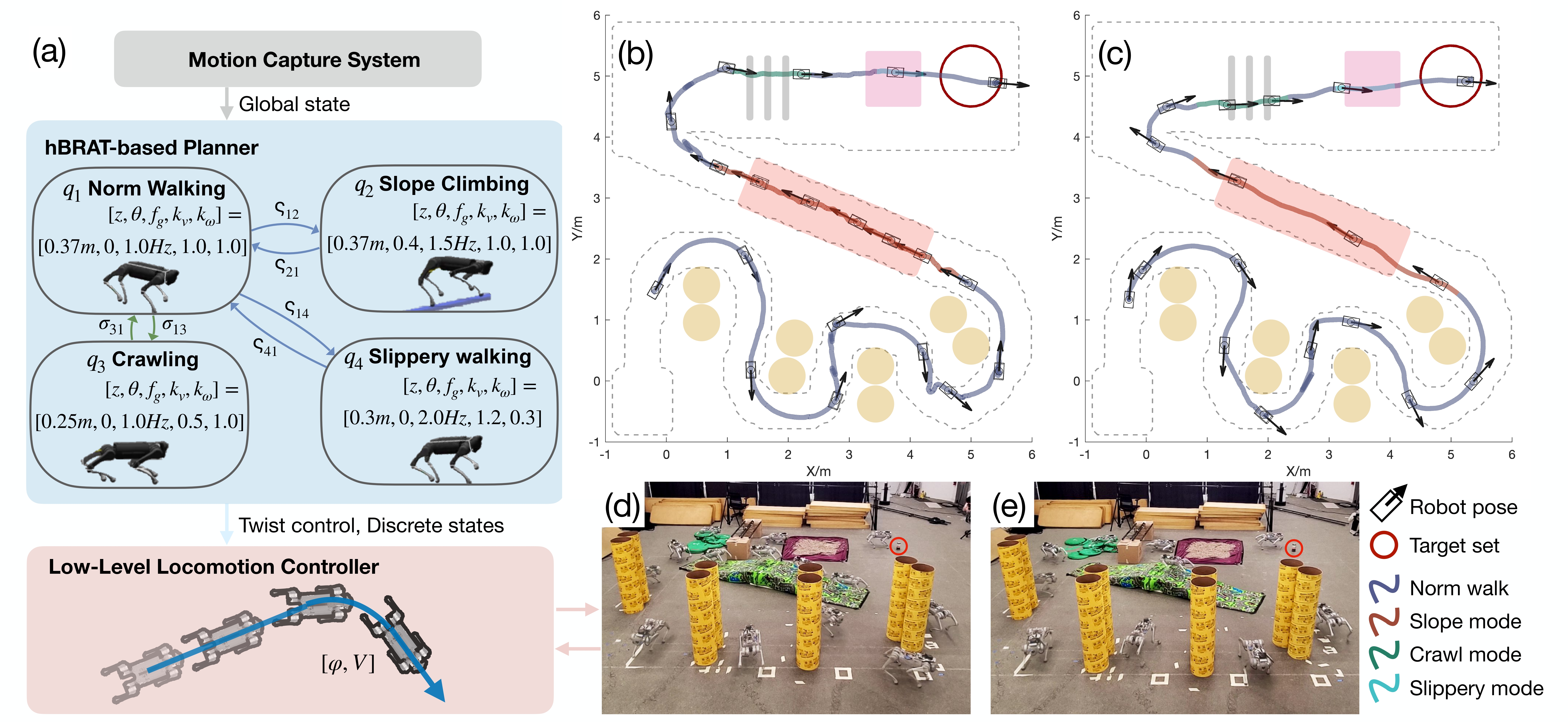}
\caption{(a) The proposed control pipeline system diagram. Top: The high-level Hybrid Reach-Avoid Planner leverages a precalculated hBRAT to derive optimal continuous control and discrete dynamic mode with estimated states from the motion capture system at run time. Bottom: Low-level locomotion policy tracks velocity command and discrete states from the hBRAT-based planner with joint-level control. (b, c) Top-down view of the robot’s trajectory during the Barkour task with both MPC controller (b) and RL policy (c). Discrete modes are marked with different trajectory colors. (d, e) Overlaid real-world trajectory of the robot for both MPC controller (d) and RL policy (e).}
\label{fig:hw_main}
\end{figure*}

\section{Full System Experiments}
\label{cases}

 To illustrate feasibility of real-time online execution of the proposed hBRAT framework proposed in Section~\ref{hybrid_BRAT}, we test it on a real-world quadruped experiment in the adapted Barkour \citep{caluwaerts2023barkourbenchmarkinganimallevelagility} benchmarking environment inspired by dog agility competitions. The environment consists of a diverse set of obstacles as shown in Fig.~\ref{fig:hw_main}(d-e), to test a wide range of locomotion behaviors:
(1) weave poles that require tight turning;
(2) a slope to test the quadruped’s climbing gait;
(3) an overhang obstacle that requires crawling;
(4) a slippery terrain with additional environmental uncertainty.
The quadruped needs to navigate from a designed start area to the goal area while completing each obstacle. 

To generate safe and diverse locomotion for the benchmark task, we designed a hierarchical control structure comprised of two components: (1) a high-level hBRAT-based planner, which takes in the robot’s current state and leverages a precalculated hBRAT to reason about both optimal discrete mode and continuous control; (2) a low-level locomotion controller which tracks robot’s target pose and velocity generated by the high-level planner. 

We formulate the planner task as a reach-avoid problem and compute the hBRAT over time. Therefore, we are able to query the optimal continuous control and discrete mode from the hBRAT at each timestep with the robot’s state. We assume the planner has \textit{a priori} knowledge of environmental obstacles, which are subsequently mapped into the predefined failure set $\mathcal{L}$. The high-level reach-avoid planner uses the following simplified continuous dynamics to describe torso’s evolution under a twist-commanded control:

\vspace{-0.4em}
\begin{equation}
\begin{aligned}
\dot{x} &= k_{v}V \sin \varphi + d_x, \quad
\dot{y} = k_{v}V \cos \varphi + d_y, \\
\dot{\varphi} &= k_{\omega} \omega_{N}, \quad
\dot{V} = f_N / m,\quad ||d_x,d_y||\leq0.5~\mathrm{m/s}.
\end{aligned}
\end{equation}

The simplified continuous dynamics reduce the state dimension, making hBRAT computation tractable. This is necessary due to the curse of dimensionality in grid-based HJ methods, whose computational cost scales exponentially with state dimension. Consequently, reduced-order models are commonly used for practical reach-avoid analysis in robotics \citep{bansal2017hamilton}.

Continuous state $[x,y,\varphi, V]^T$ represents the torso’s global $xy$ position, heading angle, and forward velocity, $m$ is the torso's mass. Controls $[\omega_{N}, f_{N}]^T$ are the angular velocity and forward force inputs. Disturbance $[d_x,d_y]^T$, bounded based on empirical estimation, captures unmodeled dynamics, environmental interactions, and mismatch between the planned high-level motion and the motion executed by the low-level locomotion controller. To describe the torso height, pitch angle, gait frequency, linear velocity gain, and angular velocity gain, we use discrete states $[z,\theta,f_{g},k_{v}, k_{\omega}]^T$. These are low-frequency states that remain fixed in each discrete mode, as shown in Fig.~\ref{fig:hw_main}(a).

In our hybrid model, each locomotion mode is associated with a particular obstacle type. Mode transitions use a mix of controlled and forced switches as a modeling choice, although for the tasks considered both lead to the same behavior. For example, the transition from walking $q_1$ to crawling $q_3$ under an overhang is modeled as controlled; however, any non-crawling mode would result in collision, while outside the overhang normal walking is always optimal due to its higher speed. As a result, the reach-avoid optimization selects crawling only within the overhang region. Transitions for slope climbing ($q_2$) and slippery terrain ($q_4$) are encoded as forced to explicitly demonstrate the alternative hybrid modeling formulation. This mixed representation emphasizes that mode selection is ultimately governed by environment-induced feasibility and optimality.

The hBRAT for this task is calculated on a $[x,y,\varphi, V]$ grid of $[70,70,40,5]$ points over a time horizon of $T=90$~s. We performed this one-time offline computation on an AMD Ryzen 9 4800HS CPU with 16 GB RAM, which took approximately 7 hours. During deployment, the robot inferences online from the pre-computed hBRAT with real-world state to determine the optimal control action and switching logic. The inference is computationally light, taking around 1ms, with an AMD Ryzen 5 5500U single-board computer and 16GB of RAM.
The platform is a 12-DOF Unitree Go1 quadruped tracked by a Vicon motion capture system at 120~Hz, providing global state estimation. Both our high-level planner and low-level locomotion run on a separate PC mounted on Go1 and interact with Go1’s low-level motor controller through wired Ethernet.

\vspace{-1em}

\subsection{Simulation Validation}

Before proceeding to hardware experiments and in order to validate our approach, we conduct a simulation study comparing our method against sampling-based, deterministic, and safety-critical MPC baselines.  All planners operate on the same reduced-order dynamics and task setups.

We evaluate three baselines \textbf{(a) MPPI}: a Model Predictive Path Integral planner for hybrid dynamical systems following \citet{Parwana2025hybridMPPI}. The planner samples $M = 128$ control rollouts around a nominal sequence with additive Gaussian noise ($\sigma^2 = 100$) over a $T = 0.4$~s horizon, combined via softmin weighting with temperature $\lambda = 2.0$. \textbf{(b) Deterministic MPC}: a direct multiple-shooting nonlinear optimization over a $T = 2.5$~s horizon ($N = 25$ at $\Delta t = 0.1$~s) solved at each control step in CasADi~\citep{andersson2018casadi}, tracking a precomputed reach trajectory reference with hard signed-distance constraints. \textbf{(c) CBF-MPC}: the deterministic MPC of (b) augmented with discrete-time control barrier function (CBF) constraints following \citet{zeng2021safety}. With $h(x) = \mathrm{sdf}(x)$, we enforce $h_{k+1} \ge (1-\gamma)\, h_k$ at every node with discrete-time CBF decay rate $\gamma = 0.3 \in (0, 1]$.

All three baselines were extensively tuned for a fair comparison. Their cost pairs a shared collision penalty with a progress term towards the goal. The resulting high-level twist commands are tracked by a learned locomotion controller, which maps twist commands to joint-level actions.

In simulation, the quadruped is tasked with reaching a target region while avoiding the same obstacle configuration used in the real world test. To evaluate robustness, we increase task difficulty by applying torso velocity disturbances of growing magnitude. For each disturbance level, we run 100 trials per planner in MuJoCo~\citep{todorov2012mujoco}, terminating each trial on goal entry or after a 40-second timeout. We report success reach rate, safe reach rate, time-to-goal over successful runs, and obstacle clearance metrics based on minimum and average signed distance (SD). Quantitative results are summarized in Table \ref{tab:sim_5metric}, with representative trajectories in Fig.~\ref{fig:hbrat_mppi_sim_comp}.

\begin{table}[ht]
\centering
\caption{Simulation results for the hBRAT planner and the MPPI, MPC, and MPC-CBF baselines across disturbance levels (100 trials per level). Values are reported as mean $\pm$ standard deviation; $\dagger$ and $\ddagger$ indicate statistically significant differences ($p<0.01$ and $p<0.001$). \textbf{Succ.\ Reach} is the fraction of trials that reach the goal; \textbf{Safe Reach} is the fraction of \emph{reached} trials without collisions, where a collision is any timestep when the signed distance to obstacles becomes negative; \textbf{T$_{\text{goal}}$} is the time to goal for successful trials; and \textbf{Min SD} / \textbf{Avg SD} are the minimum and mean signed distance to obstacles.}
\label{tab:sim_5metric}
\setlength{\tabcolsep}{3.0pt}
\renewcommand{\arraystretch}{1.10}
\scriptsize
\resizebox{\columnwidth}{!}
{
\begin{tabular}{c l|c c c c c}
\toprule
  \makecell{\textbf{Dist.}\\\textbf{(m/s)}}
& \textbf{Planner}
& \makecell{\textbf{Succ.}\\\textbf{Reach}\,$\uparrow$}
& \makecell{\textbf{Safe}\\\textbf{Reach}\,$\uparrow$}
& \textbf{$T_{\text{goal}}$ (s)\,$\downarrow$}
& \makecell{\textbf{Min}\\\textbf{SD (m)}\,$\uparrow$}
& \makecell{\textbf{Avg}\\\textbf{SD (m)}\,$\uparrow$} \\
\midrule
 \multirow{4}{*}{0.0} & hBRAT   & $\mathbf{1.00}$ & $\mathbf{1.00}$ & $22.91 \pm 0.20$ & $\mathbf{0.08 \pm 0.02}$ & $0.35 \pm 0.01$ \\
  & MPPI    & $0.95$ & $0.02$ & $37.46 \pm 10.31^{\ddagger}$ & $-0.38 \pm 0.23^{\ddagger}$ & $0.14 \pm 0.14^{\ddagger}$ \\
  & MPC    & $\mathbf{1.00}$ & $0.99$ & $\mathbf{22.71 \pm 0.28}^{\ddagger}$ & $0.02 \pm 0.00^{\ddagger}$ & $\mathbf{0.38 \pm 0.00}^{\ddagger}$ \\
  & MPC-CBF & $\mathbf{1.00}$ & $\mathbf{1.00}$ & $22.88 \pm 0.37 ^{\ddagger}$ & $0.02 \pm 0.00^{\ddagger}$ & $\mathbf{0.38 \pm 0.01}^{\ddagger}$ \\
\midrule
 \multirow{4}{*}{0.2} & hBRAT   & $\mathbf{1.00}$ & $\mathbf{0.99}$ & $\mathbf{23.08 \pm 0.35}$ & $\mathbf{0.07 \pm 0.02}$ & $0.35 \pm 0.01$ \\
  & MPPI    & $0.93$ & $0.17$ & $36.33 \pm 10.37^{\ddagger}$ & $-0.32 \pm 0.27^{\ddagger}$ & $0.16 \pm 0.13^{\ddagger}$ \\
  & MPC    & $\mathbf{1.00}$ & $0.84$ & $23.61 \pm 1.03^{\ddagger}$ & $0.01 \pm 0.01^{\ddagger}$ & $\mathbf{0.37 \pm 0.01}^{\ddagger}$ \\
  & MPC-CBF & $\mathbf{1.00}$ & $0.90$ & $23.52 \pm 0.96^{\dagger}$ & $0.02 \pm 0.01^{\ddagger}$ & $\mathbf{0.37 \pm 0.02}^{\ddagger}$ \\
\midrule
 \multirow{4}{*}{0.4} & hBRAT   & $0.99$ & $\mathbf{0.81}$ & $\mathbf{23.58 \pm 0.52}$ & $\mathbf{0.04 \pm 0.05}$ & $0.34 \pm 0.02$ \\
  & MPPI    & $0.89$ & $0.15$ & $35.24 \pm 11.57^{\ddagger}$ & $-0.31 \pm 0.24^{\ddagger}$ & $0.17 \pm 0.14^{\ddagger}$ \\
  & MPC    & $\mathbf{1.00}$ & $0.48$ & $25.41 \pm 2.17^{\ddagger}$ & $-0.01 \pm 0.04^{\ddagger}$ & $0.35 \pm 0.02^{\ddagger}$ \\
  & MPC-CBF & $\mathbf{1.00}$ & $0.44$ & $25.69 \pm 2.51^{\ddagger}$ & $-0.01 \pm 0.04^{\ddagger}$ & $\mathbf{0.36 \pm 0.02}^{\ddagger}$ \\
\bottomrule
\end{tabular}
}
\end{table}

Table \ref{tab:sim_5metric} shows that hBRAT planner achieves the highest safe reach rate at every disturbance level. At $\sigma = 0.4$~m/s it reaches $0.81$, while MPC, MPC-CBF, and MPPI drop to $0.48$, $0.44$, and $0.15$. Differences between hBRAT and each baseline on $T_{\text{goal}}$, Min SD, and Avg SD are statistically significant under the Mann-Whitney $U$ test. The deterministic MPC baselines match hBRAT on success reach, but their safe reach performance degrades sharply as the disturbance grows. Since their safety enforcement is local to the prediction horizon, MPC-CBF performs similarly to deterministic MPC on safety, as its discrete CBF condition is slack-softened to retain feasibility under fast disturbances. MPPI’s finite-horizon, soft-constrained setup results in the weakest robustness, with safety violations on most trials.  As Fig.~\ref{fig:hbrat_mppi_sim_comp} shows, hBRAT recovers from disturbances and returns to a near-optimal path to the goal, while baselines reach the goal but repeatedly penetrate the obstacle set. 

\begin{figure}[h] 
\begin{center} 

\includegraphics[width=0.99\columnwidth]{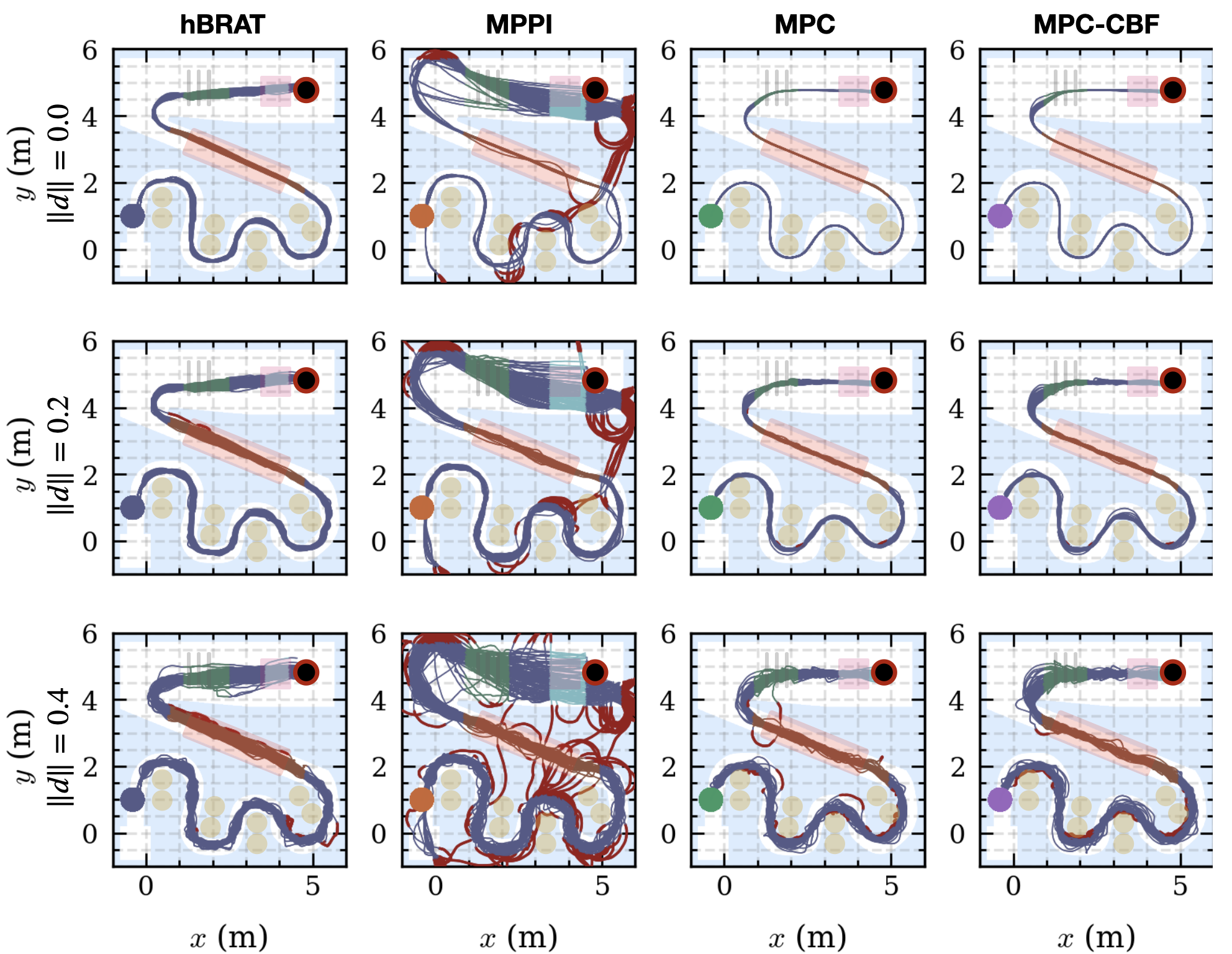}
\vspace{-0.5em}
\caption{Representative trajectories in simulation for the hBRAT planner and three baselines across disturbance levels.}
\label{fig:hbrat_mppi_sim_comp}
\end{center}
\end{figure}

Notably, at the highest disturbance level ($0.4$~m/s), hBRAT's success reach rate drops slightly from $1.00$ to $0.99$ and its safe reach rate drops more substantially to $0.81$. Both effects stem from the same underlying causes. This disturbance level sits near the assumed bound used in the hBRAT computation ($0.5$~m/s), where the injected disturbance combined with model mismatch between the planner and the locomotion controller can exceed that bound breaking the assumptions tied to the proposed method guarantees, leading to a rare violation of the underlying reachability assumptions and, in turn, occasional collisions. These collisions are shallow, with peak penetration averaging $0.048$~m, which lies within the discretization margin of error arising from both the initial discretization of the obstacle set, and the additional error introduced by the HJ level-set solver on our grid with spatial resolution $\Delta x = 0.1$~m~\citep{mitchell2008flexible}.

\subsection{Real World Safety Across Locomotion Policies}

Having validated the planner’s behavior in a controlled pre-hardware study, we now evaluate its performance on the real robot quadruped.
For all experiments, we use two different low-level locomotion controllers: (1) a model-based controller (MPC) \citep{kim2019highly}; (2) a reinforcement learning (RL) policy \citep{margolis2022walktheseways}. Both low-level controllers take continuous velocity twist commands and discrete states as inputs to track the high-level planner.

Our experiment results are shown in Fig.~\ref{fig:hw_main} and the accompanying video. In Fig.~\ref{fig:hw_main}(b) and (c), the quadruped begins in normal walking mode to navigate the weave poles (purple trajectory), then switches to slope climbing mode near the incline (dark red). As the robot approaches the overhang, the planner selects a transition to crawl-walking mode (green), lowering the body to maintain clearance and avoid collision. After clearing the overhang, it resumes normal walking. On the slippery gravel terrain, the system triggers a forced transition to a higher-frequency gait to prevent slipping (light blue). These discrete mode transitions demonstrate the framework’s ability to reason about safe and optimal behaviors in complex environments.

The robot's velocity during these two runs is shown in Fig.~\ref{fig:hw_vel_map}. In both cases, the robot completes all the obstacles without collisions. For the MPC controller, the quadruped completes the course in $58.6$~s, whereas the RL policy achieves a faster traversal time of $17.9$~s. The MPC controller deployed here relies on the linearized centroidal dynamics model, simplifying the robot as a rigid body with reduced degrees of freedom to ensure optimization efficiency. However, it inherently limits the controller's performance under highly dynamic situations. The hBRAT-based planner ensures safety even when the MPC controller has limitations in tracking its velocity commands. In comparison, the RL policy is trained with whole-body dynamics, enabling better performance when tracking velocity commands from the hBRAT-based planner.

\begin{figure}[b] 
\begin{center} 
\includegraphics[width=0.99\columnwidth]{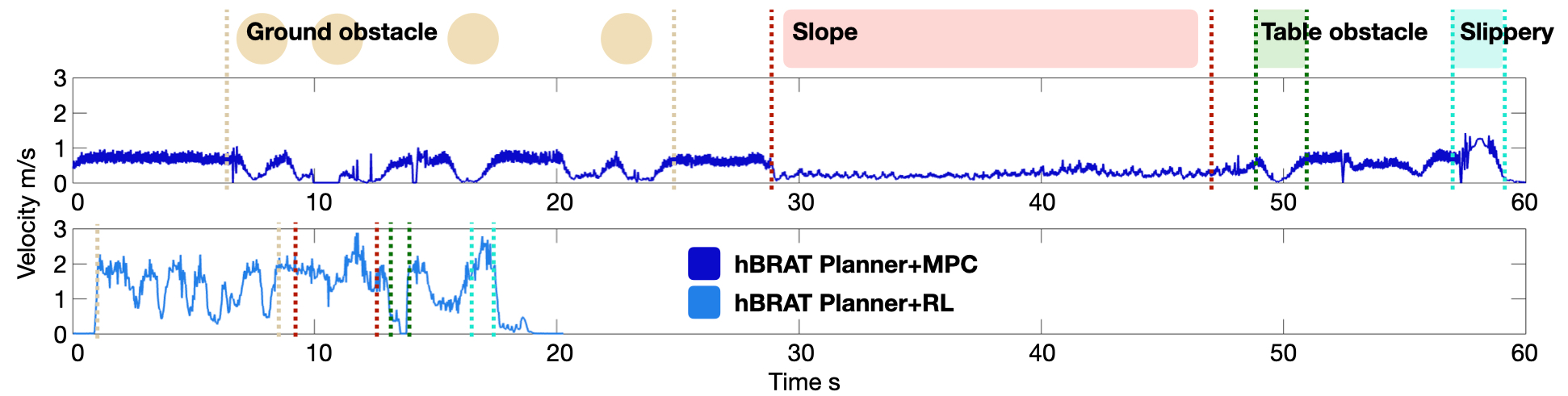}
\vspace{-0.5em}
\caption{Robot's velocity during a single task for both low-level locomotion controllers.}
\label{fig:hw_vel_map}
\vspace{-1em}
\end{center}
\end{figure}

From both subplots, one can observe that the robot autonomously reduces its speed when navigating tight turns around the weave poles, ensuring a smaller turning radius. Furthermore, the robot slows down when crawling through the overhang obstacle with a turning velocity constraint. These behaviors showcase that the hBRAT-based planner consistently provides optimal continuous control within each discrete dynamic mode. 

\begin{table}[t]
\centering
\caption{Hardware results over 10 trials with MPC and RL locomotion controllers. \textbf{Succ.\ Reach} is the fraction of trials that reach the goal; \textbf{Safe Reach} is the fraction of \emph{reached} trials without collisions; and \textbf{Min SD} / \textbf{Avg SD} are the minimum and mean signed distance to obstacles.}
\label{tab:hw_5metric}
\setlength{\tabcolsep}{4.0pt}
\renewcommand{\arraystretch}{1.10}
\scriptsize
\resizebox{\columnwidth}{!}
{
\begin{tabular}{l|c c c c}
\toprule
\makecell{\textbf{Low}\\\textbf{Level}}
  & \makecell{\textbf{Succ.}\\\textbf{Reach}\,$\uparrow$}
  & \makecell{\textbf{Safe}\\\textbf{Reach}\,$\uparrow$}
  & \textbf{Min SD (m)\,$\uparrow$}
  & \textbf{Avg SD (m)\,$\uparrow$} \\
\midrule
  MPC & 1.00 & 1.00  & $0.05 \pm 0.04$ & $0.27 \pm 0.03$ \\
  RL  & 1.00 & 1.00  & $0.12 \pm 0.11$ & $0.29 \pm 0.02$ \\
\bottomrule
\end{tabular}
}
\end{table}

Table~\ref{tab:hw_5metric} reports results over 10 trials on hardware, where each trial introduces disturbances, including loose slippery terrain, and human disturbances attempting to push or halt the robot. We omit time-to-goal because operator-injected disturbances have arbitrary timing and duration, making $T_{\text{goal}}$ not a clean controller metric. The successful navigation across two different low-level locomotion controllers highlights the robustness of the hBRAT-based planner. Regardless of whether the underlying locomotion controller is model-based or learned, the hBRAT-based planner is able to provide long-horizon safety guarantees while enabling the low-level controller to execute the optimal locomotion behavior in complex environments. This under the assumption that unmodeled dynamics, environmental interactions, state estimation error, and mismatch between the planned high-level dynamics and the motion executed by the low-level controller, remains within the worst-case bounded additive disturbances modeled in the planner, and that grid resolution and numerical integration error of the underlying value function remain small enough not to compromise safety margins.

\subsection{Robustness to External Disturbances}

To further evaluate the close-loop robustness of our method, we introduced random environmental disturbances by placing slippery material before the overhang obstacle (green plastic film). These changes were not encoded in the failure set when computing the hBRAT. Notably, the framework remained robust to these dynamic uncertainties with both low-level locomotion controllers, which was once again guaranteed by the hBRAT under the assumption that the overall disturbances lie within the bounds of the adversarial disturbance considered during its calculation.

In addition to the slippery terrain, we introduced random force disturbances by kicking and dragging the robot during the experiment, as indicated by red arrows in Fig.~\ref{fig:hmn_dist_traj}. In Fig.~\ref{fig:hmn_dist_traj}(b), the robot's torso was pulled backward, pushing it closer to an overhanging obstacle. In response, the hBRAT-based planner ensured safety by reactivating the crawl-walking mode, demonstrating its closed-loop ability to reason over both continuous and discrete control decisions to maintain safety under unexpected disturbances.

The robot can also recover when dragged into unsafe regions, as shown in Fig.~\ref{fig:hmn_dist_traj}(c), which was not possible using our preliminary framework~\citep{borquez2024hybridreach}. In our previous quadruped experiments, entering an unsafe region triggered a forced transition into frozen dynamics, enforcing a reach-avoid behavior while using a purely reach hBRT formulation. While this approach kept the robot away from hazards, it lacked recovery capabilities once inside an unsafe region, whether due to intentional disturbances or state estimation errors. By leveraging the reach-avoid formulation, the hBRAT-based planner provides a fallback safety control strategy, allowing the robot to actively recover even if it momentarily enters an unsafe region.

\begin{figure}[t] 
\begin{center} 

\includegraphics[width=0.99\columnwidth]{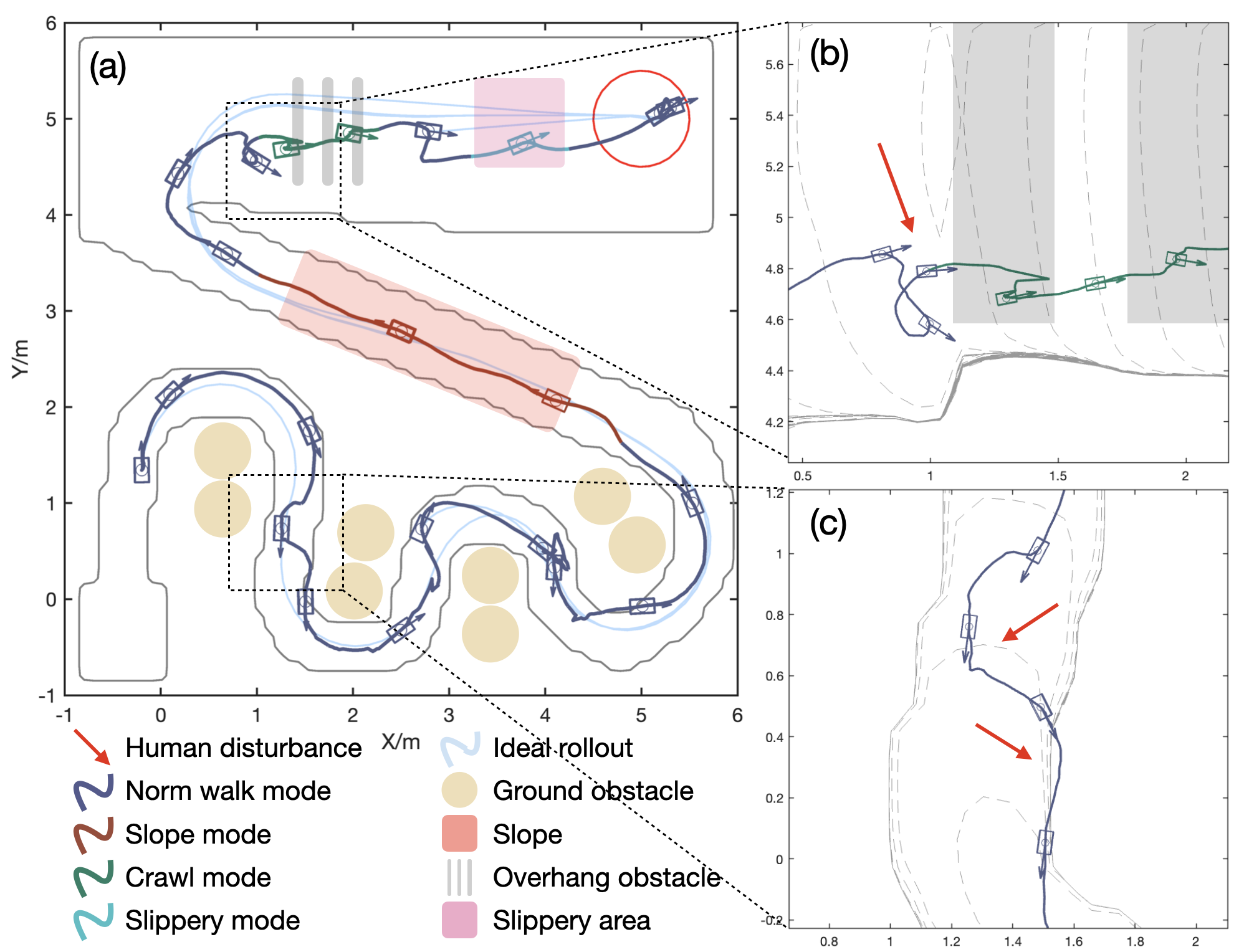}
\vspace{-0.5em}
\caption{(a) Robot trajectory with human disturbances (red arrow). (b) Partial trajectory shows the robot autonomously switches from $q_1$ to $q_3$ 
after being pushed into the overhang obstacle area to ensure safety. (c) Partial trajectory shows the robot recovers from the boundary of the obstacle grids.}
\label{fig:hmn_dist_traj}
\vspace{-1em}
\end{center}
\end{figure}

The proposed framework integrates the high-level reach-avoid planner with low-level locomotion controllers, enabling safe and efficient real-world navigation in the Barkour environment setup. The hierarchical structure of our approach allows seamless integration of different low-level controllers, including both model-based and learned policies, as demonstrated in our experiments. Nevertheless, it should be emphasized that the presented framework only provides the top level of a hierarchical planning architecture. Additional modules such as environmental perception, real-time state estimation, and disturbance estimation play a vital role in achieving safe and agile locomotion in real-world environments \citep{miki2022learning}. We will further explore integrating our framework within such modules in the future.

%% file: 11_conclusion.tex
\section{\label{conclusion}Discussion And Future Work}
\vspace{0em}

In this work, we extend the classical Hamilton-Jacobi (HJ) reachability framework to hybrid dynamical systems, introducing a generalized methodology for computing hybrid backward reachable tubes and hybrid backward reach-avoid tubes. By leveraging these reachable sets, we develop a hybrid least restrictive safety filter, which minimally modifies a nominal hybrid policy to prevent unsafe transitions, as well as a hybrid reach-avoid control framework, which simultaneously guarantees both safety and task completion. We demonstrate the effectiveness of these methods through extensive simulations and real-world hardware experiments on quadrupedal robots, showcasing their applicability in contact-rich and dynamic environments.

A key challenge in hybrid reachability remains scalability to higher-dimensional systems. Our framework currently relies on grid-based numerical methods, whose computational complexity scales exponentially with the number of continuous states, limiting direct applicability to high-dimensional hybrid systems. Future work will explore recent learning-based approaches for solving high-dimensional HJI-VI or structure-exploiting techniques such as Sum-of-Squares programming for polynomial dynamics and Zonotope methods for affine dynamics. These methods offer a promising direction for improving computational efficiency while maintaining rigorous safety guarantees.

Another limitation of our current approach is the reliance on a priori knowledge of the full environment to compute the hybrid BRT or BRAT offline before deployment. In many real-world scenarios, robotic systems operate in partially known or dynamically changing environments, where precomputing the full reachable set is impractical. Future work could explore incremental or online reachability methods, where reachable sets are progressively refined as new information about the environment becomes available. Additionally, an exciting direction is the extension of our framework to black-box hybrid systems, where the dynamics, mode domains, or transition surfaces are not explicitly defined but instead learned from data. By leveraging model learning techniques or data-driven approximations of reachable sets, it may be possible to apply reachability-based safety filtering and reach-avoid control even in settings where an analytical representation of the hybrid system is unavailable.

A fundamental consideration in hybrid systems is the prevention of Zeno behaviors. Our current computation algorithm avoids Zeno behaviors by using a finite timestep $\delta$ for discrete transitions (automatically selected by the Level Set Toolbox to balance computational efficiency with approximation fidelity). However, this comes at a price of approximate value functions and it would be interesting to explore how to properly handle Zeno behaviors during the reachable set and safety controller computations in Algorithm \ref{alg1} and Algorithm \ref{alg2}.

Finally, applying the proposed framework to a broader range of robotic systems presents an exciting opportunity. While we validate our approach on quadrupedal locomotion, hybrid reachability and safety filtering have applications in legged manipulation, multi-contact locomotion, and aerial-ground hybrid vehicles. Extending our framework to these domains would further demonstrate the versatility and impact of reachability-based hybrid safety guarantees, enabling robots to operate in increasingly complex and uncertain environments.

%% file: main.bbl
\begin{thebibliography}{60}
\providecommand{\natexlab}[1]{#1}
\providecommand{\url}[1]{\texttt{#1}}
\providecommand{\urlprefix}{URL }
\expandafter\ifx\csname urlstyle\endcsname\relax
  \providecommand{\doi}[1]{DOI:\discretionary{}{}{}#1}\else
  \providecommand{\doi}{DOI:\discretionary{}{}{}\begingroup \urlstyle{rm}\Url}\fi

\bibitem[{Althoff and Krogh(2012)}]{zono_avoid_intersect_2012}
Althoff M and Krogh BH (2012) Avoiding geometric intersection operations in reachability analysis of hybrid systems.
\newblock In: \emph{Proceedings of the 15th ACM international conference on Hybrid Systems: Computation and Control}.

\bibitem[{Althoff et~al.(2010)Althoff, Stursberg and Buss}]{zono_poly_2010}
Althoff M, Stursberg O and Buss M (2010) Computing reachable sets of hybrid systems using a combination of zonotopes and polytopes.
\newblock \emph{Nonlinear Analysis: Hybrid Systems} IFAC World Congress 2008.

\bibitem[{Alt{\i}n and Sanfelice(2020)}]{altin2020semicontinuity}
Alt{\i}n B and Sanfelice RG (2020) Semicontinuity properties of solutions and reachable sets of nominally well-posed hybrid dynamical systems.
\newblock In: \emph{2020 59th IEEE conference on decision and control (CDC)}. IEEE.

\bibitem[{Alur et~al.(1991)Alur, Courcoubetis, Henzinger and Ho}]{alur1991hybrid}
Alur R, Courcoubetis C, Henzinger TA and Ho PH (1991) Hybrid automata: An algorithmic approach to the specification and verification of hybrid systems.
\newblock In: \emph{International Hybrid Systems Workshop}. Springer.

\bibitem[{Alur and Dill(1994)}]{alur1994theory}
Alur R and Dill DL (1994) A theory of timed automata.
\newblock \emph{Theoretical computer science} .

\bibitem[{Andersson et~al.(2018)Andersson, Gillis, Horn, Rawlings and Diehl}]{andersson2018casadi}
Andersson J, Gillis J, Horn G, Rawlings J and Diehl M (2018) Casadi—a software framework for nonlinear optimization and optimal control.
\newblock \emph{Mathematical Programming Computation} 11(1): 1--36.

\bibitem[{Bansal et~al.(2017)Bansal, Chen, Herbert and Tomlin}]{bansal2017hamilton}
Bansal S, Chen M, Herbert S and Tomlin CJ (2017) {Hamilton-Jacobi Reachability}: A brief overview and recent advances.
\newblock In: \emph{IEEE Conference on Decision and Control (CDC)}.

\bibitem[{Bansal et~al.(2020)Bansal, Chen, Tanabe and Tomlin}]{bansal2020provably}
Bansal S, Chen M, Tanabe K and Tomlin CJ (2020) Provably safe and scalable multivehicle trajectory planning.
\newblock \emph{IEEE Transactions on Control Systems Technology (TCST)} .

\bibitem[{Bansal and Tomlin(2021)}]{bansal2021deepreach}
Bansal S and Tomlin CJ (2021) {DeepReach}: A deep learning approach to high-dimensional reachability.
\newblock In: \emph{IEEE International Conference on Robotics and Automation (ICRA)}.

\bibitem[{Benvenuti et~al.(2008)Benvenuti, Bresolin, Casagrande, Collins, Ferrari, Mazzi, Sangiovanni-Vincentelli and Villa}]{ariadne_denotable_reach08}
Benvenuti L, Bresolin D, Casagrande A, Collins P, Ferrari A, Mazzi E, Sangiovanni-Vincentelli A and Villa T (2008) Reachability computation for hybrid systems with ariadne.
\newblock \emph{IFAC Proceedings Volumes} .

\bibitem[{Borquez et~al.(2024{\natexlab{a}})Borquez, Chakraborty, Wang and Bansal}]{borquezFiltering2023}
Borquez J, Chakraborty K, Wang H and Bansal S (2024{\natexlab{a}}) On safety and liveness filtering using hamilton–jacobi reachability analysis.
\newblock \emph{IEEE Transactions on Robotics} \doi{10.1109/TRO.2024.3454470}.

\bibitem[{Borquez et~al.(2024{\natexlab{b}})Borquez, Peng, Chen, Nguyen and Bansal}]{borquez2024hybridreach}
Borquez J, Peng S, Chen Y, Nguyen Q and Bansal S (2024{\natexlab{b}}) {Hamilton-Jacobi Reachability Analysis for Hybrid Systems with Controlled and Forced Transitions}.
\newblock In: \emph{Proceedings of Robotics: Science and Systems}. Delft, Netherlands.
\newblock \doi{10.15607/RSS.2024.XX.006}.

\bibitem[{Breeden and Panagou(2023)}]{Satellite_CBF}
Breeden J and Panagou D (2023) Safety-critical control for systems with impulsive actuators and dwell time constraints.
\newblock \emph{IEEE Control Systems Letters} .

\bibitem[{Bui et~al.(2022)Bui, Giovanis, Chen and Shriraman}]{bui2022optimizeddp}
Bui M, Giovanis G, Chen M and Shriraman A (2022) Optimizeddp: An efficient, user-friendly library for optimal control and dynamic programming.
\newblock \emph{arXiv preprint arXiv:2204.05520} .

\bibitem[{Caluwaerts et~al.(2023)Caluwaerts, Iscen, Kew, Yu, Zhang, Freeman, Lee, Lee, Saliceti, Zhuang, Batchelor, Bohez, Casarini, Chen, Cortes, Coumans, Dostmohamed, Dulac-Arnold, Escontrela, Frey, Hafner, Jain, Jyenis, Kuang, Lee, Luu, Nachum, Oslund, Powell, Reyes, Romano, Sadeghi, Sloat, Tabanpour, Zheng, Neunert, Hadsell, Heess, Nori, Seto, Parada, Sindhwani, Vanhoucke and Tan}]{caluwaerts2023barkourbenchmarkinganimallevelagility}
Caluwaerts K, Iscen A, Kew JC, Yu W, Zhang T, Freeman D, Lee KH, Lee L, Saliceti S, Zhuang V, Batchelor N, Bohez S, Casarini F, Chen JE, Cortes O, Coumans E, Dostmohamed A, Dulac-Arnold G, Escontrela A, Frey E, Hafner R, Jain D, Jyenis B, Kuang Y, Lee E, Luu L, Nachum O, Oslund K, Powell J, Reyes D, Romano F, Sadeghi F, Sloat R, Tabanpour B, Zheng D, Neunert M, Hadsell R, Heess N, Nori F, Seto J, Parada C, Sindhwani V, Vanhoucke V and Tan J (2023) Barkour: Benchmarking animal-level agility with quadruped robots.
\newblock \urlprefix\url{https://arxiv.org/abs/2305.14654}.

\bibitem[{Chai and Sanfelice(2018)}]{chai2018forward}
Chai J and Sanfelice RG (2018) Forward invariance of sets for hybrid dynamical systems (part i).
\newblock \emph{IEEE Transactions on Automatic Control} .

\bibitem[{Chen et~al.(2012)Chen, Abraham and Sankaranarayanan}]{taylor_hyb_flow12}
Chen X, Abraham E and Sankaranarayanan S (2012) Taylor model flowpipe construction for non-linear hybrid systems.
\newblock In: \emph{2012 IEEE 33rd Real-Time Systems Symposium}. IEEE.

\bibitem[{Chen et~al.(2013)Chen, {\'A}brah{\'a}m and Sankaranarayanan}]{flow_star_2013}
Chen X, {\'A}brah{\'a}m E and Sankaranarayanan S (2013) Flow*: An analyzer for non-linear hybrid systems.
\newblock In: \emph{Computer Aided Verification: 25th International Conference, CAV 2013, Saint Petersburg, Russia, July 13-19, 2013. Proceedings 25}. Springer.

\bibitem[{Choi et~al.(2022)Choi, Agrawal, Sreenath, Tomlin and Bansal}]{ROA_reset_2022}
Choi JJ, Agrawal A, Sreenath K, Tomlin CJ and Bansal S (2022) Computation of regions of attraction for hybrid limit cycles using reachability: An application to walking robots.
\newblock \emph{IEEE Robotics and Automation Letters} \doi{10.1109/LRA.2022.3151143}.

\bibitem[{Choi et~al.(2021)Choi, Lee, Sreenath, Tomlin and Herbert}]{cbvf}
Choi JJ, Lee D, Sreenath K, Tomlin CJ and Herbert SL (2021) Robust control barrier–value functions for safety-critical control.
\newblock In: \emph{2021 60th IEEE Conference on Decision and Control (CDC)}.
\newblock \doi{10.1109/CDC45484.2021.9683085}.

\bibitem[{Collins and Lygeros(2005)}]{finite_time_denotable_2005}
Collins P and Lygeros J (2005) Computability of finite-time reachable sets for hybrid systems.
\newblock In: \emph{Proceedings of the 44th IEEE Conference on Decision and Control}.
\newblock \doi{10.1109/CDC.2005.1582902}.

\bibitem[{Curtis(2019)}]{curtis2019orbital}
Curtis HD (2019) \emph{Orbital mechanics for engineering students}.
\newblock Butterworth-Heinemann.

\bibitem[{Damm et~al.(2007)Damm, Disch, Hungar, Jacobs, Pang, Pigorsch, Scholl, Waldmann and Wirtz}]{large_discrete_2007}
Damm W, Disch S, Hungar H, Jacobs S, Pang J, Pigorsch F, Scholl C, Waldmann U and Wirtz B (2007) Exact state set representations in the verification of linear hybrid systems with large discrete state space.
\newblock In: \emph{International Symposium on Automated Technology for Verification and Analysis}. Springer.

\bibitem[{Dhinakaran et~al.(2017)Dhinakaran, Chen, Chou, Shih and Tomlin}]{coll_avoid_HJI_hybrid_2017}
Dhinakaran A, Chen M, Chou G, Shih JC and Tomlin CJ (2017) A hybrid framework for multi-vehicle collision avoidance.
\newblock In: \emph{Conference on Decision and Control}.

\bibitem[{Fisac et~al.(2015)Fisac, Chen, Tomlin and Sastry}]{Fisac15}
Fisac J, Chen M, Tomlin CJ and Sastry S (2015) {Reach-avoid problems with time-varying dynamics, targets and constraints}.
\newblock In: \emph{HSCC}.

\bibitem[{Fisac et~al.(2019)Fisac, Lugovoy, Rubies-Royo, Ghosh and Tomlin}]{fisac2019bridging}
Fisac JF, Lugovoy NF, Rubies-Royo V, Ghosh S and Tomlin CJ (2019) Bridging hamilton-jacobi safety analysis and reinforcement learning.
\newblock In: \emph{IEEE International Conference on Robotics and Automation}.

\bibitem[{Gillula et~al.(2011)Gillula, Hoffmann, Huang, Vitus and Tomlin}]{drone_backflip_2011}
Gillula JH, Hoffmann GM, Huang H, Vitus MP and Tomlin CJ (2011) Applications of hybrid reachability analysis to robotic aerial vehicles.
\newblock \emph{International Journal of Robotics Research} .

\bibitem[{Girard(2013)}]{girard2013computational}
Girard A (2013) \emph{Computational approaches to analysis and control of hybrid systems}.
\newblock PhD Thesis, Universit{\'e} de Grenoble.

\bibitem[{Goebel et~al.(2009)Goebel, Sanfelice and Teel}]{goebel2009hybrid}
Goebel R, Sanfelice RG and Teel AR (2009) Hybrid dynamical systems.
\newblock \emph{IEEE control systems magazine} .

\bibitem[{Henzinger et~al.(1995{\natexlab{a}})Henzinger, Ho and Wong-Toi}]{henzinger1995user}
Henzinger TA, Ho PH and Wong-Toi H (1995{\natexlab{a}}) A user guide to hytech.
\newblock In: \emph{International Workshop on Tools and Algorithms for the Construction and Analysis of Systems}. Springer.

\bibitem[{Henzinger et~al.(1995{\natexlab{b}})Henzinger, Kopke, Puri and Varaiya}]{henzinger1995s}
Henzinger TA, Kopke PW, Puri A and Varaiya P (1995{\natexlab{b}}) What's decidable about hybrid automata?
\newblock In: \emph{Proceedings of the twenty-seventh annual ACM symposium on Theory of computing}.

\bibitem[{Herbert(2020)}]{herbert2020safe}
Herbert SL (2020) \emph{Safe real-world autonomy in uncertain and unstructured environments}.
\newblock University of California, Berkeley.

\bibitem[{Johnson et~al.(2016)Johnson, Burden and Koditschek}]{johnson2016hybrid}
Johnson AM, Burden SA and Koditschek DE (2016) A hybrid systems model for simple manipulation and self-manipulation systems.
\newblock \emph{The International Journal of Robotics Research} .

\bibitem[{Kim et~al.(2019)Kim, Carlo, Katz, Bledt and Kim}]{kim2019highly}
Kim D, Carlo JD, Katz B, Bledt G and Kim S (2019) Highly dynamic quadruped locomotion via whole-body impulse control and model predictive control.

\bibitem[{Kochdumper and Althoff(2020)}]{taylor_nonlin_guard_20}
Kochdumper N and Althoff M (2020) Reachability analysis for hybrid systems with nonlinear guard sets.
\newblock HSCC '20. Association for Computing Machinery.
\newblock ISBN 9781450370189.

\bibitem[{Kong et~al.(2021)Kong, Payne, Council and Johnson}]{kong2021salted}
Kong NJ, Payne JJ, Council G and Johnson AM (2021) The salted kalman filter: Kalman filtering on hybrid dynamical systems.
\newblock \emph{Automatica} .

\bibitem[{Kong et~al.(2015)Kong, Gao, Chen and Clarke}]{delta_reach2015}
Kong S, Gao S, Chen W and Clarke E (2015) dreach: $\delta$-reachability analysis for hybrid systems.
\newblock In: \emph{Tools and Algorithms for the Construction and Analysis of Systems: 21st International Conference, TACAS 2015, Held as Part of the European Joint Conferences on Theory and Practice of Software, ETAPS 2015, London, UK, April 11-18, 2015, Proceedings 21}. Springer.

\bibitem[{Lygeros(1996)}]{lygeros1996hierarchical}
Lygeros J (1996) \emph{Hierarchical, hybrid control of large-scale systems}.
\newblock University of California, Berkeley.

\bibitem[{Lygeros(2004)}]{lygeros2004reachability}
Lygeros J (2004) On reachability and minimum cost optimal control.
\newblock \emph{Automatica} .

\bibitem[{Lygeros et~al.(1998)Lygeros, Tomlin and Sastry}]{lygeros1998controller}
Lygeros J, Tomlin C and Sastry S (1998) On controller synthesis for nonlinear hybrid systems.
\newblock In: \emph{Proceedings of the 37th IEEE Conference on Decision and Control (Cat. No. 98CH36171)}. IEEE.

\bibitem[{Lygeros et~al.(2008)Lygeros, Tomlin and Sastry}]{lygeros2008hybrid}
Lygeros J, Tomlin C and Sastry S (2008) Hybrid systems: modeling, analysis and control.
\newblock \emph{Electronic Research Laboratory, University of California, Berkeley, CA, Tech. UCB/ERL M} .

\bibitem[{Maghenem and Sanfelice(2019)}]{HCBF_inclusion}
Maghenem M and Sanfelice RG (2019) Characterizations of safety in hybrid inclusions via barrier functions.
\newblock In: \emph{Proceedings of the 22nd ACM International Conference on Hybrid Systems: Computation and Control}.

\bibitem[{Maiga et~al.(2015)Maiga, Ramdani, Trav{\'e}-Massuy{\`e}s and Combastel}]{nonlin_zono_2015}
Maiga M, Ramdani N, Trav{\'e}-Massuy{\`e}s L and Combastel C (2015) A comprehensive method for reachability analysis of uncertain nonlinear hybrid systems.
\newblock \emph{IEEE Transactions on Automatic Control} .

\bibitem[{Maler et~al.(1995)Maler, Pnueli and Sifakis}]{maler1995synthesis}
Maler O, Pnueli A and Sifakis J (1995) On the synthesis of discrete controllers for timed systems.
\newblock In: \emph{STACS 95: 12th Annual Symposium on Theoretical Aspects of Computer Science Munich, Germany, March 2--4, 1995 Proceedings 12}. Springer.

\bibitem[{Margolis and Agrawal(2022)}]{margolis2022walktheseways}
Margolis GB and Agrawal P (2022) Walk these ways: Tuning robot control for generalization with multiplicity of behavior.
\newblock \emph{Conference on Robot Learning} .

\bibitem[{Miki et~al.(2022)Miki, Lee, Hwangbo, Wellhausen, Koltun and Hutter}]{miki2022learning}
Miki T, Lee J, Hwangbo J, Wellhausen L, Koltun V and Hutter M (2022) Learning robust perceptive locomotion for quadrupedal robots in the wild.
\newblock \emph{Science robotics} .

\bibitem[{Mitchell(2004)}]{mitchell2004toolbox}
Mitchell I (2004) A toolbox of level set methods.
\newblock \url{https://www.cs.ubc.ca/~mitchell/ToolboxLS/toolboxLS-1.1.pdf, Tech. Rep. TR-2004-09}.

\bibitem[{Mitchell et~al.(2005)Mitchell, Bayen and Tomlin}]{mitchell2005time}
Mitchell I, Bayen A and Tomlin CJ (2005) A time-dependent {Hamilton-Jacobi} formulation of reachable sets for continuous dynamic games.
\newblock \emph{IEEE Transactions on Automatic Control (TAC)} .

\bibitem[{Mitchell and Tomlin(2000)}]{air_3modes_200l}
Mitchell I and Tomlin CJ (2000) Level set methods for computation in hybrid systems.
\newblock In: \emph{International workshop on hybrid systems: Computation and control}. Springer.

\bibitem[{Mitchell(2008)}]{mitchell2008flexible}
Mitchell IM (2008) The flexible, extensible and efficient toolbox of level set methods.
\newblock \emph{Journal of Scientific Computing} 35(2): 300--329.

\bibitem[{Parwana et~al.(2025)Parwana, Black, Hoxha, Okamoto, Fainekos, Prokhorov and Panagou}]{Parwana2025hybridMPPI}
Parwana H, Black M, Hoxha B, Okamoto H, Fainekos G, Prokhorov D and Panagou D (2025) Risk-aware mppi for stochastic hybrid systems.
\newblock In: \emph{2025 American Control Conference (ACC)}.

\bibitem[{Robey et~al.(2021)Robey, Lindemann, Tu and Matni}]{HCBF_learn}
Robey A, Lindemann L, Tu S and Matni N (2021) Learning robust hybrid control barrier functions for uncertain systems.
\newblock \emph{IFAC-PapersOnLine} .

\bibitem[{Schmerling and Pavone(2023)}]{hj_reach_ASL2023}
Schmerling E and Pavone M (2023) hj reachability: Hamilton-jacobi reachability analysis in jax.
\newblock \urlprefix\url{https://github.com/StanfordASL/hj_reachability}.

\bibitem[{Tang and Althoff(2023)}]{contact_zono_2023}
Tang C and Althoff M (2023) Formal verification of robotic contact tasks via reachability analysis.
\newblock \emph{arXiv preprint arXiv:2307.13977} .

\bibitem[{Todorov et~al.(2012)Todorov, Erez and Tassa}]{todorov2012mujoco}
Todorov E, Erez T and Tassa Y (2012) Mujoco: A physics engine for model-based control.
\newblock In: \emph{2012 IEEE/RSJ international conference on intelligent robots and systems}. IEEE.

\bibitem[{Tomlin et~al.(1999)Tomlin, Lygeros and Sastry}]{air_7modes_1999}
Tomlin C, Lygeros J and Sastry S (1999) Computing controllers for nonlinear hybrid systems.
\newblock In: \emph{Hybrid Systems: Computation and Control: Second International Workshop, HSCC’99 Berg en Dal, The Netherlands, March 29--31, 1999 Proceedings 2}. Springer.

\bibitem[{Tomlin et~al.(1996)Tomlin, Pappas, Lygeros, Godbole, Sastry and Meyer}]{tomlin1996hybrid}
Tomlin C, Pappas G, Lygeros J, Godbole D, Sastry S and Meyer G (1996) Hybrid control in air traffic management systems.
\newblock \emph{IFAC Proceedings Volumes} .

\bibitem[{Yang et~al.(2024)Yang, Black, Fainekos, Hoxha, Okamoto and Mangharam}]{HCBF_filter}
Yang S, Black M, Fainekos G, Hoxha B, Okamoto H and Mangharam R (2024) Safe control synthesis for hybrid systems through local control barrier functions.
\newblock In: \emph{2024 American Control Conference (ACC)}.

\bibitem[{Yovine(1997)}]{yovine1997kronos}
Yovine S (1997) Kronos: A verification tool for real-time systems.
\newblock \emph{Int. J. Softw. Tools Technol. Transf.} .

\bibitem[{Zeng et~al.(2021)Zeng, Zhang and Sreenath}]{zeng2021safety}
Zeng J, Zhang B and Sreenath K (2021) Safety-critical model predictive control with discrete-time control barrier function.
\newblock In: \emph{2021 American control conference (ACC)}. IEEE, pp. 3882--3889.

\end{thebibliography}
